\documentclass[jair,twoside,11pt]{article}
\usepackage{jair, rawfonts} 

\usepackage{theapa}
\usepackage[utf8]{inputenc}
\usepackage{graphicx}
\usepackage{array}
\usepackage{subcaption}
\usepackage{booktabs}
\usepackage[figuresright]{rotating}
\usepackage{placeins}
\usepackage[ruled,linesnumbered]{algorithm2e}
\usepackage{pgf,tikz}
\usetikzlibrary{calc}
\usetikzlibrary{arrows}
\usepackage{amssymb}
\usepackage{amsthm}
\usepackage{amsmath}
\usepackage{multirow}
\usepackage{url}
\newtheorem{theorem}{Theorem}[section]

\newtheorem{lemma}[theorem]{Lemma}
\theoremstyle{definition}
\newtheorem{definition}{Definition}[section]
\theoremstyle{remark}

\usepackage{listings}
\lstdefinestyle{customc}{language=C,
  basicstyle=\footnotesize\ttfamily,
  numbers=left,
  stepnumber=1,
  showstringspaces=false,
  tabsize=1,
  breaklines=true,
  breakatwhitespace=false,
}

\ShortHeadings{Faster and More Robust Mesh-based Algorithms for Obstacle $k$-Nearest Neighbour}
{Zhao, Taniar \& Harabor}

\begin{document}

\title{Faster and More Robust Mesh-based Algorithms for Obstacle $k$-Nearest Neighbour}
\author{
  \name Shizhe Zhao \email {szha414@student.monash.edu} \\
  \addr Monash University\\
  \addr Melbourne, Australia\\
  \AND
  \name David Taniar \email {david.taniar@monash.edu}\\
  \addr Monash University\\
  \addr Melbourne, Australia\\
  \AND
  \name Daniel D. Harabor \email {daniel.harabor@monash.edu}\\
  \addr Monash University\\
  \addr Melbourne, Australia\\
}

\maketitle
\begin{abstract}
We are interested in the problem of finding $k$ nearest neighbours in the 
plane and in the presence of polygonal obstacles (\emph{OkNN}). 
Widely used algorithms for OkNN are based on incremental
visibility graphs, which means they require costly and online visibility 
checking and have worst-case quadratic running time.
Recently \textbf{Polyanya}, a fast point-to-point pathfinding algorithm was 
proposed which avoids the disadvantages of visibility graphs by searching
over an alternative data structure known as a navigation mesh.
Previously, we adapted \textbf{Polyanya} to multi-target scenarios by 
developing two specialised heuristic functions: the \textbf{Interval
heuristic} $h_v$ and the \textbf{Target heuristic} $h_t$. Though these methods 
outperform visibility graph algorithms by orders of magnitude in all 
our experiments they are not robust: $h_v$ expands many redundant nodes when 
the set of neighbours is small while $h_t$ performs poorly when the set of 
neighbours is large.
In this paper, we propose new algorithms and heuristics for OkNN which 
perform well regardless of neighbour density.
\end{abstract}
\section{Introduction}
\label{introduction}
The Obstacle (equiv. Obstructed) k-Nearest Neighbour Problem (OkNN) is a 
generalised variant of the well known Euclidean k-Nearest Nearest Problem 
(kNN) in two dimensions.
In both cases the objective is to return the $k$ closest neighbours 
(equiv. targets) to an a priori unknown query point $q$. In the case of OkNN
however polygonal obstacles are introduced and the objective is to return
shortest distances to each of the $k$ closest neighbours without crossing any 
obstacles. We refer to this metric as the \emph{obstacle distance} between
traversable points. Fig~\ref{obs_dis} highlights the
differences between the two problem settings.

\begin{figure}[tb]
  \centering
  \includegraphics[width=.5\textwidth]{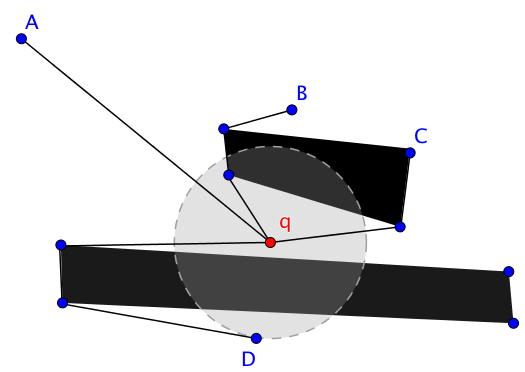}
  \caption{\small
  We aim to find the nearest neighbour of point $q$ from among the set of 
  target points $A,B,C,D$.
  Notice $D$ is the nearest neighbor of $q$ under the Euclidean metric but also the furthest neighbor of $q$ when obstacles are considered.}
\label{obs_dis}
\end{figure}


OkNN problems appear in a number of different application areas and are
of interest in several different research fields.  For example, in the context
of spatial databases, it is often desirable to perform clustering queries in
the presence of natural or man-made obstacles (such as rivers, trees or
buildings)~\cite{thh-scipo-01}.
In the context of computer games meanwhile, e.g. as described
in~\cite{amp-ltw-05}, agents often rely on nearest-neighbour information, such
as the location of enemies or resource points, for creating higher-level
plans.  Besides being a problem of direct interest, OkNN is also a fundamental 
sub-problem for a number of related spatial queries such as:
\begin{itemize}
  \item \textbf{Obstacle Range Query} (OR)~\cite{DBLP:conf/edbt/ZhangPMZ04}:
  given a query point $q$ and a range $r$, find all targets whose
  obstacle distance to $q$ at most equal to $r$. 
  OkNN algorithms can be applied to range queries by setting $k$ to be 
  infinite and terminating when the distance to the next closest target is 
  greater than $r$.

  \item \textbf{Obstacle Reverse k-Nearest Neighbor} (ORkNN) and 
  variants~\cite{DBLP:conf/gis/GaoYCZC11,DBLP:journals/isci/GaoLMY16}):
  given a query point $q$ and a value $k$, return a set of targets $t$ that
  have $q$ as one of their $k$ nearest neighbours: i.e. $\lbrace t~|~q \in OkNN(t, k)\rbrace$.
  ORkNN typically involes two stages:
    (i) \textit{search stage}, which explores the space to obtain a set of candidates;
    (ii) \textit{refine stage}, which removes from the set of candidates those 
    for which $q$ is not in the top $k$ neighbours.
    Each of the two stages benefits directly from an efficient routine for 
    OkNN.

  \item \textbf{Continuous Obstacle k-Nearest Neighbor} (COkNN)~\cite{DBLP:conf/sigmod/GaoZ09}: 
  similar to OkNN, but the query point is generalised to a line segment. 
  Solving such queries typically involves generating a list of so-called 
  ``split points'', and then invoking an OkNN query for each of these points.
\end{itemize}

Two popular algorithms for OkNN, which can deal with obstacles, are
\emph{local visibility graphs}~\cite{DBLP:conf/edbt/ZhangPMZ04} and \emph{fast
filter}~\cite{DBLP:conf/bncod/XiaHT04}. Though different in details, both of 
these methods are similar in that they depend on the incremental and online 
construction of a graph of co-visible points followed by an online Dijkstra 
search.  Algorithms of this type are simple to understand, 
provide optimality guarantees and the promise of fast performance. 
Such advantages make incremental visibility graphs attractive to researchers.
However, incremental visibility graphs also suffer from a number of notable 
disadvantages including:
(i) online visibility checks;
(ii) an incremental construction process that has up to quadratic space and 
time complexity for the worst case;
(iii) duplicated effort, since the graph is discarded each time the query 
point changes.

In a previous paper~\cite{DBLP:conf/socs/ZhaoTH18} we develop a new method for 
computing OkNN which avoids these disadvantages and which can improve runtime 
performance by several orders of magnitude.
Our work extends \textit{Polyanya}~\cite{DBLP:conf/ijcai/CuiHG17}: a recent 
and very fast algorithm for computing Euclidean shortest paths on a navigation 
mesh: a data structure comprised of
convex polygons which taken together represent the entire traversable space.  
Compared to visibility graphs, navigation meshes are much cheaper to construct 
and sometimes available as input ``for free'' (e.g. in computer game settings, 
navigation meshes are often created, at least in part, by human designers).  
In a range of experiments we show that each of our previously proposed techniques 
can be several orders of magnitude faster than \textit{LVG}~\cite{DBLP:conf/edbt/ZhangPMZ04},
a state-of-the-art algorithm based on incremental visibility graphs.

\noindent \textbf{Limitations in Our Previous Work}.
Our previous work described three multi-target variants of \textit{Polyanya},
each of which offers state of the art performance but only in specific
experimental scenarios:
\begin{itemize}
  \item \textbf{Brute-force Polyanya}, a simple algorithm that invokes
  point-to-point Polyanya repeatedly, once for each target. Though extremely
  naive this algorithm was undominated in our experiments in scenarios with
  few targets and for values of $k > 2$.
  \item \textbf{Interval Heuristic $h_v$}, which replaces the point-to-point
  heuristic function of the original algorithm with a consistent and
  online alternative that minimises cost-to-go distance for all targets in
  the candidate set. In our experiments this algorithm was undominated in
  dense-target scenarios where the map contains many nearest neighbour
  candidates.
  \item \textbf{Target Heuristic $h_t$}, a costly pre-processing-based
  heuristic function which always chooses the Euclidean-nearest candidate for
  computing cost-to-go estimates. In our experiments this algorithm was
  undominated in sparse-target scenarios with small values of $k$.
\end{itemize}

\noindent \textbf{Contributions}.
In this paper, we propose new preprocessing-based algorithms and heuristics 
for OkNN which improve on the performance of our previous work and which are 
more \emph{robust}, in the sense that they  perform well across a larger range 
of experimental scenarios including different target densities and for a 
wider range of values for $k$:
\begin{itemize}
  \item We combine Polyanya with a technique known as Incremental Euclidean 
      Restriction to derive a simple but very effective OkNN 
      algorithm based on repeatedly solving  point-to-point queries.
  \item We develop an efficient preprocessing framework that involves 
      computing labels that we store with the edges of the navigation mesh.
      We exploit these labels during a subsequent online phase and derive
      two new OkNN query algorithms:
    (i) \textbf{Fence checking} which is very fast but only works when $k=1$;
    (ii) \textbf{Fence Heuristic $h_f$} which works in the general case and
    which performs as well as and sometimes better than $h_v$ and $h_t$,
    regardless of target density.
\end{itemize}


\section{Related work}
\label{related}

\subsection{\textit{R-tree} and traditional kNN}
\label{related:s}
 Traditional kNN queries in the plane (i.e. no obstacles) is a well studied 
 problem for which there exists many well known and efficient spatial indexing
 schemes. 
 Perhaps the most successful and well known of these is the
 \textit{R-tree}~\cite{DBLP:conf/sigmod/Guttman84}, a hierarchical 
data structure which organises a collection of spatial objects 
 embedded in the plane (e.g. points or polygons) into a height-balanced
 tree.
Though many variants exist~\cite{DBLP:conf/vldb/SellisRF87,DBLP:conf/sigmod/BeckmannKSS90,DBLP:conf/vldb/KamelF94}, each of which improves the basic idea in some way,
they all operate similarly and provide similar functionality.
For this reason we opt to describe only the original work.

There are two types of nodes in an \emph{R-tree}: leaves, which represent 
single objects, and interior nodes which are associated with a 
\textit{Minimal Bounding Rectangle} (MBR). The MBR of an interior node 
contains the objects of all of its children and their associated MBRs.
Meanwhile, nearest-neighbour queries involve a branch-and-bound traversal 
that depends on two important metrics:
\begin{itemize}
  \item \textit{\textbf{mindist}}, which bounds from below the minimum 
      distance from the query point $q$ to the closest object in the MBR of 
      the current node;
  \item \textit{\textbf{minmaxdist}}, which bounds from below the minimum
      distance from the query point $q$ to the furthest object in the MBR
      of the current node. 
\end{itemize}

The query algorithm starts at the root of the tree node and proceeds down, 
prioritizing nodes by \textit{mindist} and pruning the search space by 
\textit{minmaxdist}.
Though improvements to this basic algorithm exist they employ a similar schema, 
albeit with different prioritisation and pruning strategies; 
e.g.~\cite{DBLP:conf/sigmod/RoussopoulosKV95,DBLP:journals/sigmod/CheungF98}.

\subsection{OkNN Queries and Visibility Graphs}
\label{related:v}
Solving point-to-point obstacle distance queries in main-memory is a
well studied problem for which many algorithms exist~\cite{DBLP:books/lib/Berg00}.
Optimal methods usually pre-compute a visibility graph (VG), which adds edges 
between pairs of vertices that are co-visible.
Figure~\ref{vg} shows an example.
\begin{figure*}[tb]
  \centering
  \begin{minipage}[t]{0.45\columnwidth}
  \centering
  \begin{tikzpicture}[scale=0.6]
    \input{./src/polyanya.tex}
    \drawboundary
    \drawobstacles
    \drawVG
  \end{tikzpicture}
  \caption{\small Example of a visibility graph.  
      Black lines are edges in the visibility graph.}
  \label{vg}
  \end{minipage}
  \hspace{1cm}
  \begin{minipage}[t]{0.45\columnwidth}
  \centering
  \begin{tikzpicture}[scale=0.6]
    \input{src/polyanya.tex}
    {
    \drawboundary
    \drawobstacles
    \drawmeshs
    }
  \end{tikzpicture}
  \caption{\small Example of a navigation mesh. Grey lines 
  define individual mesh polygons.}
  \label{nav}
  \end{minipage}
\end{figure*}
Once the graph is created (including all candidate target points in the case
of OkNN), queries can be resolved by way of A{*} or using Dijkstra's well 
known algorithm. In case the query point is not one of the existing vertices
of the graph, a further online insertion operation is performed before search
can begin. 
The main drawback of visibility graphs is their worst-case time and space 
complexity which can be up $n^2$ where $n$ is the number of vertices in the
map~\cite{DBLP:journals/siamcomp/GhoshM91}.
In spatial database scenarios, for example, $n$ can be more than $10,000$,
so in-main-memory approaches are not suitable.
Researchers in this field are thus motivated to design variant algorithms that 
only consider and process obstacles relevant to the current query.
Two popular methods based on this idea are \textit{Local Visibility
Graphs}(LVG)~\cite{DBLP:conf/edbt/ZhangPMZ04} and \textit{Fast Filter}~\cite{DBLP:conf/bncod/XiaHT04}. 
We have previously discussed the strengths and weaknesses of these methods in 
Section~\ref{introduction}.

\subsection{Pathfinding on Navigation Mesh}
\label{related:n}
A navigation mesh is a data structure that divides the traversable
space into a set of convex polygons. Originally proposed by 
Arkin~\cite{DBLP:journals/robotica/Arkin89}, this spatial representation
technique has been widely applied to pathfinding problems in areas such as 
robotics and computer games.
Navigation meshes can be generated easily and efficiently. For example, 
\textit{Constrained Delaunay Triangulations}~\cite{DBLP:journals/algorithmica/Chew89},
a popular starting point for many mesh-based algorithms, can be constructed in
$O(nlogn)$ time. In other settings, such as computer games, navigation meshes 
are available as input ``for free'', having been created, at least in part, 
by human designers.



Until recently pathfinding algorithms developed for navigation meshes 
have typically lacked both strong performance and optimality gurantees
~\cite{kallmann2005path,DBLP:conf/aaai/DemyenB06}. 
\textit{Polyanya}~\cite{DBLP:conf/ijcai/CuiHG17} is a recent navigation mesh
algorithm which changes the status quo by being \emph{compromise-free}:
i.e. simultaneously fast, optimal and (assuming the mesh is given as
input) entirely online.
We give a technical description of this algorithm in Section~\ref{preview}.
Our previous work on the topic of OkNN ~\cite{DBLP:conf/socs/ZhaoTH18}
extends \textit{Polyanya}, from point-to-point pathfinding to the 
multi-target case.

\section{Problem Statement} \label{prob}
OkNN is a nearest-neighbour search in two dimensions that can
be formalised as follows:
\begin{definition}{Obstacle k-Nearest Neighbour (OkNN):}
Given a set of points $T$, a set of obstacles $O$, a distinguished point $q$ and and an integer $k$: 
\textbf{return} a set $\textit{kNN} = \{t | t \in T\}$ such that $d_o(q, t) \le d_o(q, t_k)$
  for all $t \in \textit{kNN}$.
\end{definition}

\noindent Where:
\begin{itemize}
\item $O$ is a set of non-traversable polygonal obstacles.
\item $T$ is a set of traversable points called \emph{targets}.
\item $q$ is a traversable point called the \emph{query point}.
\item $k$ is an input parameter that controls the number of nearest neighbours that will be returned.
\item $d_e$ and $d_o$ are functions that measure the shortest distance between two points, as discussed below.
\item $t_k$ is the $k^{th}$ nearest neighbour of $q$.
\end{itemize}
\noindent
Stated in simple words, the objective is to find the set of $k$ targets which are closest to $q$ from among all possible candidates in $T$.
When discussing distances between two points $q$ and $t$ we distinguish between two metrics:
$d_e(q, t)$ which is the well known Euclidean metric (i.e. ``straight-line distance'') and
$d_o(q, t)$ which measures the length of a shortest path $\pi_{q, t} = \langle q, \ldots, t\rangle$
between points $q$ and $t$ such that no pairwise segment of the path intersects any point inside an obstacle (i.e. ``obstacle avoiding distance''). \\ \newline
\noindent Solution approaches for OkNN can be broadly categorised into two schemas:
\begin{itemize}
    \item \textbf{Linear search methods}, which involve the repeated 
        application of a point-to-point pathfinding algorithm, from $q$ to 
        selected $t \in T$.
        Linear search methods terminate when it can be proven that the 
        obstacle distance to all $k$ nearest neighbours has beeen found.
        Instatiations of this schema include the OkNN algorithms 
        known as \emph{LVG}~\cite{DBLP:conf/edbt/ZhangPMZ04} and 
        \emph{fast filter}~\cite{DBLP:conf/bncod/XiaHT04}.
        
    \item \textbf{Spatial partitioning methods}, which consider all candidate
        targets simultaneously and which attempt to solve a given instance 
        of kNN in a single integrated search.
        Such methods apply prioritised search (e.g. best first or
        branch-and-bound) in combination with specialised heuristics that 
        prune the set of candidates and focus attention toward the most 
        promising target first. Instantiating this schema 
        are a variety of Euclidean kNN methods such
        as \emph{R-tree}~\cite{DBLP:conf/sigmod/Guttman84} and
        \emph{Spatial KD-tree}~\cite{DBLP:conf/btw/Ooi87}.
\end{itemize}
In our previous work we describe three algorithms for OkNN: 
\emph{Brute-force Polyanya}, which instantiates linear search,
and Polyanya in combination with the heuristics $h_t$ and $h_v$,
both of which instatiate the spatial partioning search schema.

\section{Polyanya and heuristic functions: h$_p$, h$_v$, h$_t$} \label{preview}
In its canonical form Polyanya is a point-to-point Euclidean shortest path
algorithm which can be seen as an instance of A*: it
performs a best-first search using an consistent heuristic function to 
prioritise nodes for expansion. The mechanical details are however quite different.
Since we will employ a similar search methodology to Polyanya, we give 
herein a brief description of that algorithm. There are three key components:

  \begin{figure}[tb]
    \begin{minipage}[t]{0.45\columnwidth}
    \centering

    \begin{tikzpicture}[line cap=round,line join=round,>=triangle 45,x=1.5cm,y=1.5cm, scale=0.7]
    \definecolor{ffqqqq}{rgb}{1.,0.,0.}
\definecolor{ududff}{rgb}{0.30196078431372547,0.30196078431372547,1.}
\clip(-0.8,-1.6) rectangle (4.5,1.9);
\fill[line width=2.pt,fill=black,fill opacity=1.0] (1.0568852561643571,1.7160911807341181) -- (-0.05223307140302057,0.7078017920365031) -- (0.8888370247147546,-0.23326830408127103) -- (-0.8588645823611135,-0.09882971892158902) -- (-0.8084501129262327,1.2623609558201914) -- cycle;
\fill[line width=2.pt,fill=black,fill opacity=1.0] (1.823359787139688,-0.3207843311089217) -- (3.2477722134398506,0.38019461763527085) -- (3.1845504618814697,1.6117703305615074) -- (4.915447245812395,-0.5392470319934045) -- cycle;
\draw [thick] (-0.05223307140302057,0.7078017920365031)-- (3.2477722134398506,0.38019461763527085);
\draw [dash pattern=on 2pt off 2pt,domain=-0.8:1.0362283883022219] plot(\x,{(-1.0239487358475874--1.1906912514220491*\x)/-0.1473913635874673});
\draw [dash pattern=on 2pt off 2pt,domain=1.0362283883022219:4.5] plot(\x,{(-2.2639847616004545--1.1031752243943984*\x)/0.7871313988374662});
\draw [line width=2.pt,color=ffqqqq] (2.3846562728722773,0.4658802302519392)-- (0.7826045650521551,0.6249235007746435);
\begin{scriptsize}
\draw [fill=ududff] (-0.05223307140302057,0.7078017920365031) circle (2.5pt);
\draw [fill=ududff] (3.2477722134398506,0.38019461763527085) circle (2.5pt);
\draw [fill=ududff] (1.0362283883022219,-1.4239595555033202) circle (2.5pt);
\draw[color=ududff] (0.7743350405885554,-1.3148373272892926) node {\large $r$};
\draw [fill=ffqqqq] (2.3846562728722773,0.4658802302519392) circle (2.0pt);
\draw[color=ffqqqq] (2.767634409298129,0.6348131501346688) node {\large $b$};
\draw [fill=ffqqqq] (0.7826045650521551,0.6249235007746435) circle (2.0pt);
\draw[color=ffqqqq] (0.45424317116074064,0.8967064978483352) node {\large $a$};
\draw[color=ffqqqq] (1.6473128663007772,0.9112561282768722) node {\large $I$};
\end{scriptsize}
    \end{tikzpicture}
    \caption{\small Search nodes in Polyanya. Notice that the interval $I = [a, b]$ is
    a contiguous subset of points drawn from an edge of the navigation mesh.
    The corresponding root point, $r$, is either the query point itself 
    or the vertex of an obstacle. Taken together they form the search node $(I, r)$.}
    \label{snode}
    \end{minipage}
    \hspace{1cm}
    \begin{minipage}[t]{0.45\columnwidth}
      \centering
      \includegraphics[width=\linewidth]{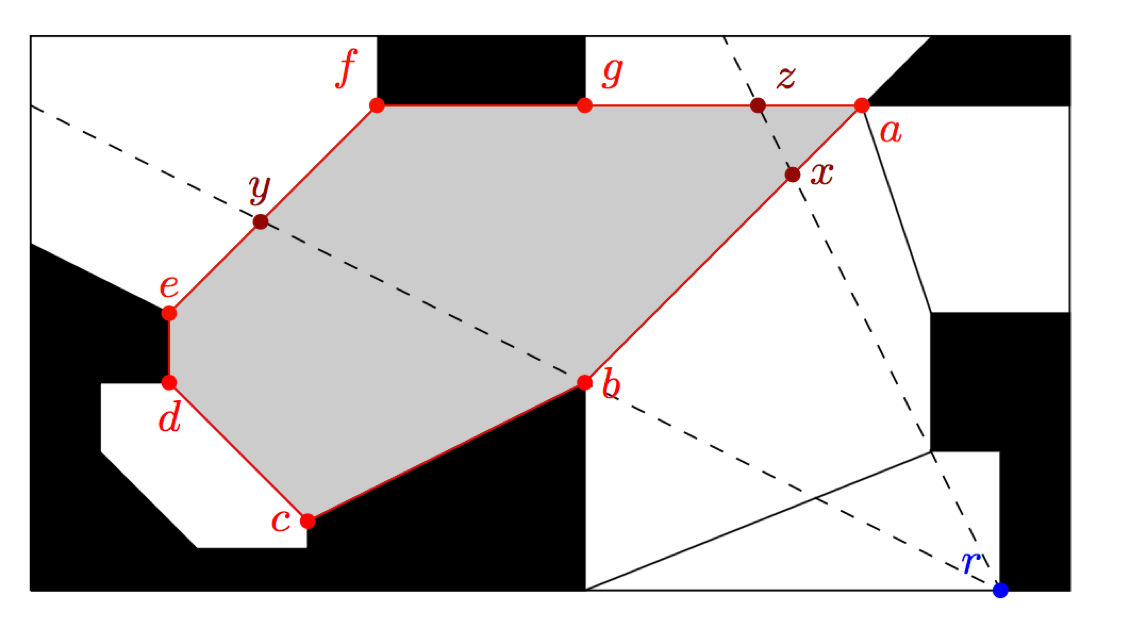}
      \caption{\small 
      From \cite{DBLP:conf/ijcai/CuiHG17}. We expand the node $([b,x],r)$ which has
      $([z,g],r)$ and $([f,y],r)$ as observable successors.
      In addition, the nodes $([c,d],b)$, $([d,e],b)$ and $([e,y],b)$ 
      are non-observable. All other potential 
      successors can be safely pruned (more details in~\cite{DBLP:conf/ijcai/CuiHG17}).}
      \label{suc}
  \end{minipage}
\end{figure}

\begin{itemize}
  \item \textbf{Search Nodes}:\ Conventional search algorithms proceed from one traversable
    point to the next. Polyanya, by comparison, searches from one \emph{edge} of 
    the navigation mesh to another. In this model search nodes are tuples $(I, r)$ 
    where each $I=[a, b]$ is a contiguous interval of points and $r$ is a distinguished
    point called the \emph{root}. Nodes are constructed such that each point $p \in I$ 
    is visible from $r$. Meanwhile, $r$ itself corresponds to the last turning point on
    the path: from $q$ to any $p \in I$.
    Fig~\ref{snode} shows an example.

  \item \textbf{Successors}:\ Successor nodes $(I', r')$ are generated by 
  ``pushing'' the current interval $I$ away from its root $r$ and through the 
  interior of an adjacent and traversable polygon.
  A successor is said to be \emph{observable} if each point $p' \in I'$ is visible
  from $r$. The successor node in this case is formed by the tuple $(I', r)$. 
  By contrast, a successor is said to be \emph{non-observable} if the \emph{taut}
  (i.e. locally optimal) path from $r$ to each $p' \in I'$ must pass through one 
  of the endpoints of current interval $I = [a, b]$. 
  The successor node in this case is formed by the tuple $(I', r')$ 
  with $r'$ as one of the points $a$ or $b$. Fig~\ref{suc} shows an example.
 
  Note that the target point is inserted in the open list as a special case
  (observable or non-observable) successor whenever the search reaches its 
  containing polygon.  The interval of this successor contains only the target.

  \item \textbf{Evaluation} using h$_p$: 
    When prioritising nodes for expansion, Polyanya makes use of an $f$-value estimation for a 
    given search node $n=(I,r)$, and target $t$:
    $$f(n)=g(n) + h_p(n, t)$$
    where $g(n)$ represents the length of a concrete shortest path from
    $q$ to $r$, and $h_p(n, t)$ represents the lower-bound from $r$ to $t$ via some $p \in I$.
    There are three cases to consider which describe the relative positions of the $t$ in
    relation to the $r$. These are illustrated in Fig~\ref{ef}. The objective in each case
    is to choose the unique $p \in I$ that minimises the estimate. The three cases together are
    sufficient to guarantee that the estimator is consistent.
\end{itemize}

Similar to A*, Polyanya terminates when the target is expanded or when the open list is empty.
In~\cite{DBLP:conf/ijcai/CuiHG17} this algorithm is shown to outpeform a range of optimal and suboptimal
competitors, often by orders of magnitude.

\begin{figure}[t]
  \centering
  \begin{tikzpicture}[line cap=round,line join=round,>=triangle 45,x=.7cm,y=.7cm, scale=0.8]
    \newcommand{\degre}{\ensuremath{^\circ}}
\definecolor{qqwuqq}{rgb}{0.,0.39215686274509803,0.}
\definecolor{ududff}{rgb}{0.30196078431372547,0.30196078431372547,1.}
\definecolor{ffqqqq}{rgb}{1.,0.,0.}
\definecolor{cqcqcq}{rgb}{0.7529411764705882,0.7529411764705882,0.7529411764705882}

\clip(-2.5,-1.5) rectangle (11.1,5.5);
\draw [thick,color=qqwuqq,fill=qqwuqq,fill opacity=0.10000000149011612] (5.882794718512914,1.3724017604956955) -- (6.1103929580172185,2.0551964790086092) -- (5.427598239504305,2.282794718512914) -- (5.2,1.6) -- cycle; 
\draw [thick,color=ffqqqq] (1.,3.)-- (10.,0.);
\draw [thick,dash pattern=on 1pt off 3pt,domain=-2.5:11.1] plot(\x,{(-42.--9.*\x)/3.});
\draw [thick] (-2.,0.)-- (1.,3.);
\draw [thick] (1.,3.)-- (0.,4.);
\draw [thick] (-2.,0.)-- (6.,4.);
\draw [thick] (4.4,-0.8)-- (2.8,2.4);
\begin{scriptsize}
\draw [fill=ffqqqq] (1.,3.) circle (2.5pt);
\draw[color=ffqqqq] (1.2251463573152723,3.6132615238618353) node {\large $a$};
\draw [fill=ffqqqq] (10.,0.) circle (2.5pt);
\draw[color=ffqqqq] (10.623297185327285,-0.2884978451684127) node {\large $b$};
\draw [fill=ududff] (-2.,0.) circle (2.5pt);
\draw[color=ududff] (-2.269472903642263,0.8989941367103583) node {\large $r$};
\draw [fill=ffqqqq] (0.,4.) circle (2.5pt);
\draw[color=ffqqqq] (-0.3864499038059212,5.190929442643631) node {\large $t_2$};
\draw [fill=ffqqqq] (6.,4.) circle (2.5pt);
\draw[color=ffqqqq] (5.347439951551588,5.021287730946663) node {\large $t_1$};
\draw[color=qqwuqq] (6.585824446939453,1.9507727492315556) node {\large $\alpha = 90\textrm{\degre}$};
\draw [fill=ffqqqq] (4.4,-0.8) circle (2.5pt);
\draw[color=ffqqqq] (3.413524438206156,-0.3393903586775029) node {\large $t_3$};
\end{scriptsize}
  \end{tikzpicture}
  \caption{\small
  Polyanya $f$-value estimator. The current node is $(I, r)$ with $I = [a, b]$ and
  each of $t_1, t_2, t_3$ are possible target locations.
  \textbf{Case 1}: the target is $t_1$. In this case the point $p \in I$ with minimum
  $f$-value is at the intersection of the interval $I$ and the line $r \rightarrow t_1$.
  \textbf{Case 2}: the target is $t_2$. In this case the $p \in I$ with minimum $f$-value
  is one of two endpoints of $I$. 
  \textbf{Case 3}: the target is $t_3$. In this case the $p \in I$ with minimum $f$-value
  is obtained by first mirroring $t_3$ through $[a, b]$ and applying Case 1 or Case 2
  to the mirrored point (here, $t_1$). Notice that in this case, simply $r$ to $t_3$
  doesn't give us the \textit{h-value}, based on Definition, it must reach the interval
  first.
    }
  \label{ef}
\end{figure}
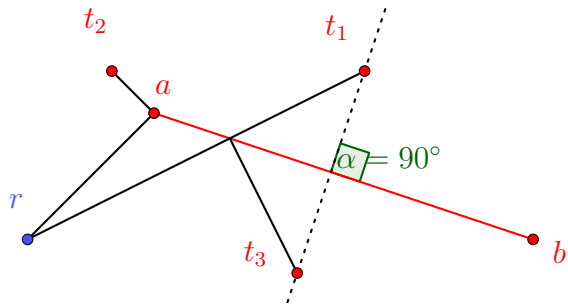

\subsection{Interval Heuristic h$_v$}
In some OkNN settings targets are myriad and one simply requires a fast 
algorithm to explore the local area. 
The idea we introduce for such settings takes the form of a specialised
heuristic function $h_v$ which we combine with \emph{Polyanya}. It can be 
formalised as follows:
\begin{definition}\label{intervalh}
  Given search node $n=(I, r)$, the interval heuristic $h_v(n)$ 
  is the minimum Euclidean distance from $r$ to any point $p \in I$.
\end{definition}
The Interval Heuristic $h_v$ is consistent~\cite{DBLP:conf/socs/ZhaoTH18}, and
applying it only requires solving a simple geometric problem: finding the
closest point on a line.
Since this operation is trivial $h_v$ has a low computation overhead and can
be applied entirely online.
Stated differently, $h_v$ discards all
spatial information about any target and performs instead a 
\textit{Dijkstra}-like search which can quickly explore surrounding polygons.
When targets are many and randomly distributed this approach can be
highly efficient.  
The main drawback is redundant search effort in non-target areas, such as can
occur when targets are few and sparsely distributed.

\subsection{Target Heuristic h$_t$}
In some OkNN settings the targets are sparse, so we need an effective method
to prune the search space. The idea we introduce for such settings again
takes the form of a specialised and consistent heuristic, $h_t$, 
which we again combine with Polyanya.
\begin{definition}\label{targeth}
  Given search node $n=(I, r)$, the target heuristic $h_t(n)$ is:
  $$
    h_t(n) = min\{h_p(n, t) | t \in T\}
  $$
\end{definition}

The idea is simple: each time Polyanya expands a node $h_t$ computes a
cost-to-go it does so with respect to the Euclidean-nearest candidate which 
has not yet been added to the $k$ nearest neighbour set.
Selecting such a target involves four nearest neighbour queries which we
solve using an \emph{R-tree}.
This approach computes all $k$ nearest neighbours in a single run. 
Compared to $h_v$, this heuristic can significantly reduce the number of 
node expansions because it always drives the search in the direction of the
most promising candidate point.
This approach performs well when target are sparse.
The main drawback is the high cost of computing nearest neighbours
when targets are numerous, dense and found in any direction.

\section{Proposed Methods}\label{proposed}
We propose two distinct ideas for improving OkNN Search. 
The first idea involves computing, during an offline pre-processing
step, a set of ``fence labels'' for each edge of the navigation mesh. 
The labels bound the distance from the edge to the nearest point in the 
target candidates set. We exploit the labels to improve the performance of a 
subsequent online search.
The second idea we propose involves running point-to-point Polyanya, from
the query location and to each target in the candidate set, in the order 
specified by a preprocessing-based heuristic technique known as Incremental 
Euclidean Restriction.

\subsection{Fence Labelling}
\label{fence}

Interval heuristic $h_v$ can give us the minimum obstacle distance from the root to an edge
of the mesh, an interesting observation is that if the algorithm have explored the entire map,
then each edge of meshes must contains such information.
Furthermore, if we ran such full-map exploration starting with each target and store obtained
information, we would utilize such information for actual queries.

The described idea cannot be efficient, in spite of the time complexity,
we need at least $O(ET)$ space to store obtained information, where $E$ is the number of mesh edges,
and $T$ is the number targets. 

However, similar to traditional nearest-neighbour query~\cite{DBLP:conf/sigmod/RoussopoulosKV95,DBLP:journals/sigmod/CheungF98},
we can apply \textit{minmaxdist} pruning in this case. 
\begin{definition}\label{mindist:def}
  Given a mesh edge $(A,B)$, a search node $n=(I,r)$ where $I=(a,b)$ and $a$ (resp. $b$) is the endpoint of
  $I$ closer to $A$ (resp. $B$), then:
  \begin{itemize}
    \item \textbf{mindist} $=h_v(r, I)$,
    \item \textbf{minmaxdist} $=max(d_e(r, a) + d_e(a, A), d_e(r, b) + d_e(b, B)$
  \end{itemize}
\end{definition}

\noindent
In other words, for a given root $r$, the length of shortest path to any point $p \in (A,B)$
must in range $[\textit{mindist}, \textit{minmaxdist}]$, see an example in Fig~\ref{mindistdef}.
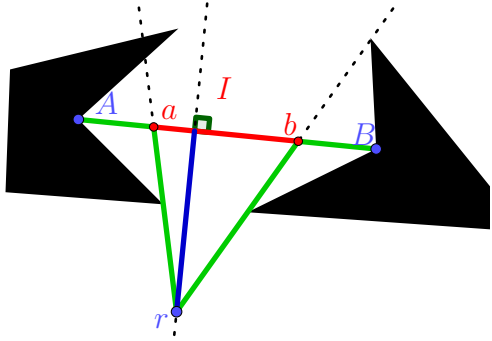
\begin{figure}[tb]
  \centering
  \begin{tikzpicture}[line cap=round,line join=round,>=triangle 45,x=1.5cm,y=1.5cm, scale=0.8]
  \definecolor{qqwuqq}{rgb}{0.,0.39215686274509803,0.}
\definecolor{qqqqcc}{rgb}{0.,0.,0.8}
\definecolor{qqccqq}{rgb}{0.,0.8,0.}
\definecolor{ffqqqq}{rgb}{1.,0.,0.}
\definecolor{ududff}{rgb}{0.30196078431372547,0.30196078431372547,1.}
\clip(-1.,-1.7) rectangle (4.5,2.);
\fill[line width=2.pt,fill=black,fill opacity=1.0] (1.0568852561643571,1.7160911807341181) -- (-0.05223307140302057,0.7078017920365031) -- (0.8888370247147546,-0.23326830408127103) -- (-0.8588645823611135,-0.09882971892158902) -- (-0.8084501129262327,1.2623609558201914) -- cycle;
\fill[line width=2.pt,fill=black,fill opacity=1.0] (1.823359787139688,-0.3207843311089217) -- (3.2477722134398506,0.38019461763527085) -- (3.1845504618814697,1.6117703305615074) -- (4.915447245812395,-0.5392470319934045) -- cycle;
\draw[line width=2.pt,color=qqwuqq,fill=qqwuqq,fill opacity=0.10000000149011612] (1.3876314031203134,0.564859617311542) -- (1.4027669225714887,0.7173205463769963) -- (1.2503059935060343,0.7324560658281715) -- (1.2351704740548592,0.5799951367627172) -- cycle; 
\draw [line width=1.pt,dash pattern=on 1pt off 4pt] (-0.05223307140302057,0.7078017920365031)-- (3.2477722134398506,0.38019461763527085);
\draw [line width=1.pt,dash pattern=on 1pt off 4pt,domain=-1.0:1.0362283883022219] plot(\x,{(-1.0239487358475874--1.1906912514220491*\x)/-0.1473913635874673});
\draw [line width=1.pt,dash pattern=on 1pt off 4pt,domain=1.0362283883022219:4.5] plot(\x,{(-2.2639847616004545--1.1031752243943984*\x)/0.7871313988374662});
\draw [line width=2.pt,color=ffqqqq] (2.3846562728722773,0.4658802302519392)-- (0.7826045650521551,0.6249235007746435);
\draw [line width=2.pt,color=qqccqq] (-0.05223307140302057,0.7078017920365031)-- (0.7826045650521551,0.6249235007746435);
\draw [line width=2.pt,color=qqccqq] (0.7826045650521551,0.6249235007746435)-- (1.0362283883022219,-1.4239595555033202);
\draw [line width=2.pt,color=qqccqq] (1.0362283883022219,-1.4239595555033202)-- (2.3846562728722773,0.4658802302519392);
\draw [line width=2.pt,color=qqccqq] (2.3846562728722773,0.4658802302519392)-- (3.2477722134398506,0.38019461763527085);
\draw [line width=1.pt,dash pattern=on 1pt off 4pt,domain=-1.:4.5] plot(\x,{(-3.88605852414162--3.300005284842871*\x)/0.3276071744012322});
\draw [line width=2.pt,color=qqqqcc] (1.0362283883022219,-1.4239595555033202)-- (1.2351704740548592,0.5799951367627172);
\begin{scriptsize}
\draw [fill=ududff] (-0.05223307140302057,0.7078017920365031) circle (2.5pt);
\draw[color=ududff] (0.2599036164566295,0.8865426152147402) node {\large $A$};
\draw [fill=ududff] (3.2477722134398506,0.38019461763527085) circle (2.5pt);
\draw[color=ududff] (3.1088661424584783,0.547089518409523) node {\large $B$};
\draw [fill=ududff] (1.0362283883022219,-1.4239595555033202) circle (2.5pt);
\draw[color=ududff] (0.8715874037520594,-1.5257410939968035) node {\large $r$};
\draw [fill=ffqqqq] (2.3846562728722773,0.4658802302519392) circle (2.0pt);
\draw[color=ffqqqq] (2.292734228862956,0.662648019449597) node {\large $b$};
\draw [fill=ffqqqq] (0.7826045650521551,0.6249235007746435) circle (2.0pt);
\draw[color=ffqqqq] (0.9575857469170178,0.7854289268046755) node {\large $a$};
\draw[color=ffqqqq] (1.563271191047489,1.0598803667748513) node {\large $I$};
\end{scriptsize}
  \end{tikzpicture}
  \caption{$A,B$ are endpoints of the mesh edge, the length of blue path is the $mindist$, the length of one of green
  path is the $minmaxdist$}
  \label{mindistdef}
\end{figure}

\noindent
To distinguish the \textit{g-}value in preprocessing and search, let's define $g_p$ as follow:
\begin{definition}\label{gp:def}
  Let $n=(I,r)$ be a search node produced by target $t$ in preprocessing:
  $$
    g_p(r) = d_o(r, t)
  $$
\end{definition}
\begin{definition}
  \textbf{Dominate}: Given two search nodes $n_1=(I_1,r_1)$ and $n_2=(I_2,r_2)$, where $I_1$
  and $I_2$ are on same mesh edge $(A, B)$,
  $n_1$ dominates $n_2$ \textit{iff} $g_p(r_1) + minmaxdist(n_1) \le g_p(r_2) + mindist(n_2)$.
\end{definition}
\begin{lemma}\label{fpruning}
  Let $n_1=(I_1,r_1)$ and $n_2=(I_2,r_2)$ be two search nodes in preprocessing stage,
  where $I_1$ and $I_2$ are on same mesh edge $(A,B)$, and $n_1$ dominates $n_2$;
  for a given query point $q$, let $m$ be any point on $(A,B)$. Then:
  $$
    d_o(q, m) + d_o(m, r_1) + g_p(r_1) \le d_o(q, m) + d_o(m, r_2) + g_p(r_2) 
  $$ 
\end{lemma}
\begin{proof}
  We have two equations: 
  \begin{equation}\label{eq1}
    g_p(r_1) + d_o(q, m) + d_o(m, r_1)
  \end{equation}
  \vspace{-1em} 
  \begin{equation}\label{eq2}
    g_p(r_2) + d_o(q, m) + d_o(m, r_2) 
  \end{equation}
  then $(\ref{eq1}) - (\ref{eq2})$ is: 
  $$
    (g_p(r_1) + d_o(m, r_1)) - (g_p(r_2) + d_o(m, r_2)) 
  $$
  according to definition~\ref{mindist:def}, we have:
  \begin{equation}
    g_p(r_1) + minmaxdist(n_1) \ge g_p(r_1) + d_o(m, r_1) 
  \end{equation}
  \begin{equation}
    g_p(r_2) + d_o(m, r_2) \ge g_p(r_2) + mindist(n_2)
  \end{equation}
  and because $n_1$ dominates $n_2$, we have:
  \begin{equation}
    g_p(r_1) + d_o(m, r_1) \le g_p(r_2) + d_o(m, r_2)
  \end{equation}
  thus $(\ref{eq1}) - (\ref{eq2}) \le 0$.
\end{proof}
Lemma~\ref{fpruning} implies that, in preprocessing, we can discard those search nodes who is
dominated by others, let's call this \textit{fence pruning}.
\begin{definition}
  \textbf{Fence} lies on each mesh edge, it contains a collection of search nodes $N$
  produced by preprocessing, let's call those search nodes \textbf{labels}
  \footnote{To make discussion clear, we call this "search node" in context of preprocessing,
  and "label" in context of query processing}.
  \textbf{Fence} has a property $\textbf{upper-bound} = min\{g_p(n.r) + minmaxdist(n) | n \in N\})$
  (default value is \textit{INF}), and in preprocessing stage,
  it blocks all passing search nodes $n'=(I',r')$ that:
  $$
    g_p(r') + mindist(n') \ge \textit{upper-bound}
  $$.
\end{definition}
Finally we propose a \textit{floodfill}-like algorithm:
in initialization, for each target $t \in T$, generate successors on the edges of mesh that 
contains $t$, and store those search nodes in a open-list; similar to OkNN search,
we expand search nodes $n$ in the order of their \textit{f-}value, where $f(n) = g(n) + h_v(n)$;
meanwhile, we apply \textit{fence pruning} and update \textit{upper-bound} of the fence befor
the expansion; the algorithm terminates when open-list becomes empty, see in algorithm~\ref{floodfill}.
\begin{algorithm}[tb]
  \input{./code/flood_fill.pseudo}
  \caption{Preprocessing}
  \label{floodfill}
\end{algorithm}

Notice that the total number of search nodes can still up to $O(ET)$, one case is that targets are
distributed in a single cluster, and all obstacles line in two rows,
so that all targets produce labels on mesh edges in the middle (see in Fig~\ref{pruning_obs}).
However, our experiment in section~\ref{exp} indicates that most of fence only contains 2 to 3
labels (search nodes), and there are two observations to convince us the effectiveness of
preprocessing (see in Fig~\ref{pruning_obs}):
\begin{itemize}
  \item If two targets are very close, their may produce labels on same fence,
    but their successors must be pruned by one of other when root changes
    (known as \textit{root pruning}~\cite{DBLP:conf/ijcai/CuiHG17}); 
  \item If two targets are not close, they will dominate surrounding mesh edges and block
    search nodes from further targets.
\end{itemize}
\begin{figure}[tb]
  \centering
  \begin{subfigure}[tb]{.45\textwidth}
    \centering
    \begin{tikzpicture}[line cap=round,line join=round,>=triangle 45,x=1.5cm,y=1.5cm]
    \definecolor{ududff}{rgb}{0.30196078431372547,0.30196078431372547,1.}
\clip(4.3,0.9) rectangle (8.1,4.6);
\fill[line width=2.pt,fill=black,fill opacity=1.0] (7.,4.) -- (7.,3.) -- (8.,3.) -- (8.,4.) -- cycle;
\fill[line width=2.pt,fill=black,fill opacity=1.0] (7.,2.) -- (7.,1.) -- (8.,1.) -- (8.,2.) -- cycle;
\fill[line width=2.pt,fill=black,fill opacity=1.0] (5.,4.) -- (5.,3.) -- (6.,3.) -- (6.,4.) -- cycle;
\fill[line width=2.pt,fill=black,fill opacity=1.0] (5.,2.) -- (5.,1.) -- (6.,1.) -- (6.,2.) -- cycle;
\draw [line width=2.pt] (7.,4.)-- (7.,3.);
\draw [line width=2.pt] (7.,3.)-- (8.,3.);
\draw [line width=2.pt] (8.,3.)-- (8.,4.);
\draw [line width=2.pt] (8.,4.)-- (7.,4.);
\draw [line width=2.pt] (7.,2.)-- (7.,1.);
\draw [line width=2.pt] (7.,1.)-- (8.,1.);
\draw [line width=2.pt] (8.,1.)-- (8.,2.);
\draw [line width=2.pt] (8.,2.)-- (7.,2.);
\draw [line width=2.pt] (5.,4.)-- (5.,3.);
\draw [line width=2.pt] (5.,3.)-- (6.,3.);
\draw [line width=2.pt] (6.,3.)-- (6.,4.);
\draw [line width=2.pt] (6.,4.)-- (5.,4.);
\draw [line width=2.pt] (5.,2.)-- (5.,1.);
\draw [line width=2.pt] (5.,1.)-- (6.,1.);
\draw [line width=2.pt] (6.,1.)-- (6.,2.);
\draw [line width=2.pt] (6.,2.)-- (5.,2.);
\draw [line width=2.pt] (5.,3.)-- (5.,2.);
\draw [line width=2.pt] (6.,3.)-- (6.,2.);
\draw [line width=2.pt] (7.,3.)-- (7.,2.);
\draw [line width=2.pt] (8.,3.)-- (8.,2.);
\draw [line width=1pt,dash pattern=on 1pt off 3pt,domain=4.6:8.1] plot(\x,{(--1.12-0.*\x)/0.4});
\draw [line width=1pt,dash pattern=on 1pt off 3pt,domain=4.6:8.1] plot(\x,{(--1.04-0.*\x)/0.4});
\draw [line width=1pt,dash pattern=on 1pt off 3pt,domain=4.6:8.1] plot(\x,{(--0.88-0.*\x)/0.4});
\draw (4.3,2.7) node[anchor=north west] {$\mathbf{\vdots}$};
\begin{scriptsize}
\draw [fill=ududff] (4.6,2.8) circle (2.5pt);
\draw[color=ududff] (4.45,2.918738659160123) node {\large $t_1$};
\draw [fill=ududff] (4.6,2.6) circle (2.5pt);
\draw[color=ududff] (4.45,2.6763202527311596) node {\large $t_2$};
\draw [fill=ududff] (4.6,2.2) circle (2.5pt);
\draw[color=ududff] (4.45,2.1) node {\large $t_n$};
\draw[color=black] (5.2,2.35) node {\large $e_1$};
\draw[color=black] (5.87,2.35) node {\large $e_2$};
\draw[color=black] (6.84,2.35) node {\large $e_3$};
\draw[color=black] (7.86,2.35) node {\large $e_4$};
\end{scriptsize}
    \end{tikzpicture}
    \caption{The worst case}
    \label{worstcase}
  \end{subfigure}%
  \begin{subfigure}[tb]{.45\textwidth}
    \centering
    \begin{tikzpicture}[line cap=round,line join=round,>=triangle 45,x=1.5cm,y=1.5cm]
    \definecolor{yqyqyq}{rgb}{0.8019607843137255,0.8019607843137255,0.8019607843137255}
\definecolor{ffqqqq}{rgb}{1.,0.,0.}
\definecolor{wwzzqq}{rgb}{0.4,0.6,0.}
\definecolor{ududff}{rgb}{0.30196078431372547,0.30196078431372547,1.}
\definecolor{zzttqq}{rgb}{0.6,0.2,0.}
\clip(4.15,0.9) rectangle (8.1,4.6);
\fill[line width=2.pt,fill=black,fill opacity=1.0] (7.,4.) -- (7.,3.) -- (8.,3.) -- (8.,4.) -- cycle;
\fill[line width=2.pt,fill=black,fill opacity=1.0] (7.,2.) -- (7.,1.) -- (8.,1.) -- (8.,2.) -- cycle;
\fill[line width=2.pt,fill=black,fill opacity=1.0] (5.,4.) -- (5.,3.) -- (6.,3.) -- (6.,4.) -- cycle;
\fill[line width=2.pt,fill=black,fill opacity=1.0] (5.,2.) -- (5.,1.) -- (6.,1.) -- (6.,2.) -- cycle;
\fill[line width=2.pt,color=zzttqq,fill=zzttqq,fill opacity=0.75] (6.,4.) -- (6.,3.) -- (7.,3.) -- cycle;
\fill[line width=2.pt,color=yqyqyq] (6.,3.) -- (6.,2.) -- (7.,3.) -- cycle;
\fill[line width=2.pt,color=wwzzqq,fill=wwzzqq,fill opacity=0.6000000238418579] (5.,3.) -- (6.,3.) -- (6.,2.) -- cycle;
\fill[line width=2.pt,color=wwzzqq,fill=wwzzqq,fill opacity=0.6000000238418579] (5.,3.) -- (5.,2.) -- (6.,2.) -- cycle;
\fill[line width=2.pt,color=yqyqyq] (7.,3.) -- (7.,2.) -- (6.,2.) -- cycle;
\fill[line width=2.pt,color=zzttqq,fill=zzttqq,fill opacity=0.75] (6.,4.) -- (7.,4.) -- (7.,3.) -- cycle;
\draw [line width=2.pt] (7.,4.)-- (7.,3.);
\draw [line width=2.pt] (7.,3.)-- (8.,3.);
\draw [line width=2.pt] (8.,3.)-- (8.,4.);
\draw [line width=2.pt] (8.,4.)-- (7.,4.);
\draw [line width=2.pt] (7.,2.)-- (7.,1.);
\draw [line width=2.pt] (7.,1.)-- (8.,1.);
\draw [line width=2.pt] (8.,1.)-- (8.,2.);
\draw [line width=2.pt] (8.,2.)-- (7.,2.);
\draw [line width=2.pt] (5.,4.)-- (5.,3.);
\draw [line width=2.pt] (5.,3.)-- (6.,3.);
\draw [line width=2.pt] (6.,3.)-- (6.,4.);
\draw [line width=2.pt] (6.,4.)-- (5.,4.);
\draw [line width=2.pt] (5.,2.)-- (5.,1.);
\draw [line width=2.pt] (5.,1.)-- (6.,1.);
\draw [line width=2.pt] (6.,1.)-- (6.,2.);
\draw [line width=2.pt] (6.,2.)-- (5.,2.);
\draw [line width=1.2pt] (6.,3.)-- (7.,3.);
\draw [line width=2.pt,color=zzttqq] (6.,4.)-- (6.,3.);
\draw [line width=2.pt,color=zzttqq] (6.,3.)-- (7.,3.);
\draw [line width=2.pt,color=zzttqq] (7.,3.)-- (6.,4.);
\draw [line width=2.pt,color=yqyqyq] (6.,3.)-- (6.,2.);
\draw [line width=2.pt,color=yqyqyq] (6.,2.)-- (7.,3.);
\draw [line width=2.pt,color=yqyqyq] (7.,3.)-- (6.,3.);
\draw [line width=2.pt,color=wwzzqq] (5.,3.)-- (6.,3.);
\draw [line width=2.pt,color=wwzzqq] (6.,3.)-- (6.,2.);
\draw [line width=2.pt,color=wwzzqq] (6.,2.)-- (5.,3.);
\draw [line width=2.pt,color=wwzzqq] (5.,3.)-- (5.,2.);
\draw [line width=2.pt,color=wwzzqq] (5.,2.)-- (6.,2.);
\draw [line width=2.pt,color=wwzzqq] (6.,2.)-- (5.,3.);
\draw [line width=2.pt,color=yqyqyq] (7.,3.)-- (7.,2.);
\draw [line width=2.pt,color=yqyqyq] (7.,2.)-- (6.,2.);
\draw [line width=2.pt,color=yqyqyq] (6.,2.)-- (7.,3.);
\draw [line width=2.pt,color=zzttqq] (6.,4.)-- (7.,4.);
\draw [line width=2.pt,color=zzttqq] (7.,4.)-- (7.,3.);
\draw [line width=2.pt,color=zzttqq] (7.,3.)-- (6.,4.);
\begin{scriptsize}
\draw [fill=zzttqq] (7.,4.) circle (2.5pt);
\draw [fill=zzttqq] (7.,3.) circle (2.5pt);
\draw [fill=zzttqq] (8.,3.) circle (2.5pt);
\draw [fill=zzttqq] (8.,4.) circle (2.5pt);
\draw [fill=ududff] (7.,2.) circle (2.5pt);
\draw [fill=wwzzqq] (7.,1.) circle (2.5pt);
\draw [fill=ududff] (8.,1.) circle (2.5pt);
\draw [fill=zzttqq] (8.,2.) circle (2.5pt);
\draw [fill=ududff] (5.,4.) circle (2.5pt);
\draw [fill=wwzzqq] (5.,3.) circle (2.5pt);
\draw [fill=ududff] (6.,3.) circle (2.5pt);
\draw [fill=zzttqq] (6.,4.) circle (2.5pt);
\draw [fill=ffqqqq] (4.260360186216868,2.51482150161698) circle (2.5pt);
\draw[color=ffqqqq] (4.230155967244832,2.8508722465565994) node {\large $t_1$};
\draw [fill=wwzzqq] (5.,2.) circle (2.5pt);
\draw [fill=wwzzqq] (5.,1.) circle (2.5pt);
\draw [fill=wwzzqq] (6.,1.) circle (2.5pt);
\draw [fill=wwzzqq] (6.,2.) circle (2.5pt);
\draw [fill=wwzzqq] (4.5,2.5) circle (2.5pt);
\draw[color=wwzzqq] (4.597664911290692,2.1914860664992763) node {\large $t_2$};
\draw [fill=zzttqq] (6.5,4.5) circle (2.5pt);
\draw[color=zzttqq] (6.661368981702059,4.218797454149341) node {\large $t_3$};
\end{scriptsize}
    \end{tikzpicture}
    \caption{General case}
    \label{pruning_obs}
  \end{subfigure}
  \caption{
    \textbf{(a)}: $t_1,t_2,\ldots,t_n$ have same distance to edges in the
  middle, so fence of these edges contain labels from all targets.
    \textbf{(b)}: $t_1, t_2$ are close and produce labels on same fence in green area, but
    $t_1$ has no labels after root change since all vertices are dominated by $t_2, t_3$;
    search nodes of $t_2$ (resp. $t_3$) can not reach brown (resp. green) area which is
    dominated by $t_3$ (resp. $t_2$).}
  \label{prep:observations}
\end{figure}
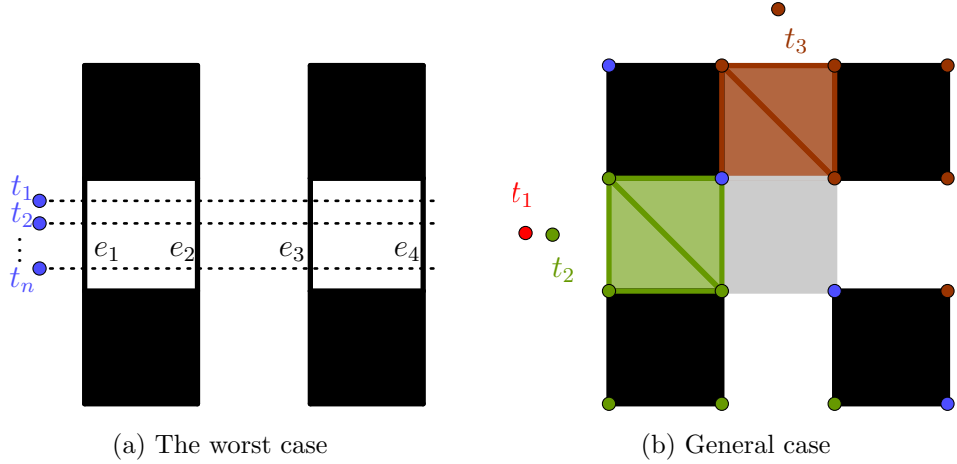

\subsection{Apply on nearest-neighbour query}
After preprocessing, for a given query point $q$, there are tow cases:
\begin{itemize}
  \item Case 1: there's no labels on surrounding fences;
  \item Case 2: all surrounding fences contain labels;
\end{itemize}
The Case 1 implies that there's no path from $q$ to any target $t \in T$;
\begin{lemma}
In the Case 2, at least one of label is produced by the obstacle nearest-neighbour of $q$.
\end{lemma}
\begin{proof}
  Assume the obstacle nearest neighbour of $q$ is $t$, and surrounding fences of $q$
  not contain any label produced by $t$, then:
  \begin{itemize}
    \item Case 1: there's no path from $t$ to $q$, which contradict to the assumption;
    \item Case 2: all successors of $t$ are dominated by others, and according
      to lemma~\ref{fpruning}, $t$ can not be the nearest-neighbour of $q$, which contradict
      to the assumption as well.
  \end{itemize}
  Thus, such $t$ doesn't exists.
\end{proof}

Therefore, we can run point-to-point search from $q$ to all roots of search nodes in
surrounding fences, and find the nearest-neighbour, we call this algorithm \textit{Fence
Checking}, see details in algorithm~\ref{fnn}.
Notice that algorithm~\ref{fnn} only search to the last vertex instead of the entire path, so it is very fast.
\begin{algorithm}[tb]
  \input{./code/fence_nn.pseudo}
  \caption{Fence Checking algorithm}
  \label{fnn}
\end{algorithm}

\subsection{Generalize to k-nearest neighbours}
When $k>1$, search nodes of next nearest-neighbour may be dominated by a visited nearest-neighbour
and thus no labels on surrounding fences. So we still need a search process to explore the map,
and we need to design heuristic function to prioritise search nodes in query processing.
We first present an admissible but not consistent heuristic \textit{naive fence
heuristic}: $h'_f$; then we show how to adapt it to a consistent heuristic
\textit{fence heuristic}: $h_f$. 
\begin{definition}\label{naive_hf:def}
  $\textbf{Naive fence heuristic}\ h'_f$: 
  let $F$ be the fence lies on mesh edge $(A,B)$, $n'=(I', r')$ be a label on $F$,
  $n=(I,r)$ be a search node produced by query point $q$, where $I$ on $(A,B)$, then
  $$
    h'_f(n) = min\{h_p(n, r') + g_p(r') | n' \in F\}
  $$
\end{definition}
Stated in simple words, $h'_f$ evaluates the distance from current root $r$ to the nearest target through
interval $I$ and $r'$.
\begin{theorem}\label{naive_hf:admis}
  $h'_f$ is admissible, more specific: given a search node $n=(I,r)$, for $\forall t' \in T$:
  $$
  f(n) = g_p(r) + h'_f(n) <= d_o(q, r) + d_e(r, m) + d_o(m, t') (m \in I)
  $$
\end{theorem}
\begin{proof}
  Given a search node $n=(I,r)$, let $n'=(I', r')$ be the label on fence $F$ that used by 
  $h'_f(n)$: $h'_f(n) = h_p(n, r') + g_p(r')$ and $n'$ is produced by a target $t'$;
  and let $p$ be the shortest path from $q$ to any target $t$ pass through $r$ and $I$;
  there are three cases:
  \begin{itemize}
    \item If $t=t'$, then we have
      $$
        |p| = g(r) + d_e(r, m) + d_o(m, r') (m \in I)
      $$
      and according to definition of $h_p$ (Fig~\ref{ef}), we have $h_p(n, r') \le d_e(r, m) + d_o(m, r') (m \in I)$,
      thus $|p| - f(n) = d_e(r, m) + d_o(m, r') - h_p(n, r') \ge 0$;
    \item If $t \neq t'$ and $t$ has label $n_t=(I_t,r_t)$ on the $F$,
      according to definition~\ref{naive_hf:def}, we have
      $f(n) <= g(r) + h_p(n, r_t) + g_p(r_t) <= |p|$;
    \item If $t \neq t'$ and none of $t$'s label on the $F$,
      then they must be dominated by some fences, according to lemma~\ref{fpruning},
      we still have $f(n) = g(r) + h'_f(n) <= |p|$.
  \end{itemize}
  Thus, $h'_f$ is admissible.
\end{proof}
Notice that $f(n) = |p|$ when $r,r'$
are co-visible, and otherwise it's an underestimation.
Such observation indicates that $h'_f$ may not consistent, Fig~\ref{naive_hf:case} shows an example.
\begin{figure}[tb]
  \centering
  \begin{tikzpicture}[line cap=round,line join=round,>=triangle 45,x=2.0cm,y=2.0cm, scale=0.7]
    \definecolor{qqqqff}{rgb}{0.,0.,1.}
\definecolor{ffqqqq}{rgb}{1.,0.,0.}
\definecolor{ududff}{rgb}{0.30196078431372547,0.30196078431372547,1.}
\clip(5.2,-0.2) rectangle (11.8,4.3);
\fill[line width=2.pt,fill=black,fill opacity=1.0] (4.676398058007273,4.660138507749294) -- (5.905109353943194,2.8603651930557095) -- (7.005965133182149,4.040058317433134) -- (6.635654371774701,5.851922676722692) -- cycle;
\fill[line width=2.pt,fill=black,fill opacity=1.0] (5.889838160743649,2.614357899610667) -- (3.8783831835080465,0.590951714017131) -- (7.00596513318215,0.) -- (6.561674304718806,0.6523407504001361) -- cycle;
\fill[line width=2.pt,fill=black,fill opacity=1.0] (8.164727588954452,4.914752704649513) -- (9.138001229812886,1.9867132398772671) -- (9.164727588954452,3.3747527046495147) -- (9.384727588954451,4.114752704649513) -- cycle;
\fill[line width=2.pt,fill=black,fill opacity=1.0] (10.30437654474037,2.3067721863002317) -- (10.270428515379974,0.5075266301992193) -- (10.762674941105722,0.20199436595565012) -- (10.813596985146315,1.7126816724932943) -- cycle;
\draw [line width=1.2pt] (7.,4.)-- (7.,-0.01697401468019827);
\draw [line width=1.2pt,dash pattern=on 1pt off 4pt,domain=5.325855281258707:11.8] plot(\x,{(--3.5661875271802543--0.20702034139792191*\x)/1.6741447187412932});
\draw [line width=1.2pt,dash pattern=on 1pt off 4pt,domain=5.325855281258707:11.8] plot(\x,{(--7.4009254403026095-0.5159119979921938*\x)/1.6685884008430216});
\draw [line width=2.pt] (4.676398058007273,4.660138507749294)-- (5.905109353943194,2.8603651930557095);
\draw [line width=2.pt] (5.905109353943194,2.8603651930557095)-- (7.005965133182149,4.040058317433134);
\draw [line width=2.pt] (7.005965133182149,4.040058317433134)-- (6.635654371774701,5.851922676722692);
\draw [line width=2.pt] (6.635654371774701,5.851922676722692)-- (4.676398058007273,4.660138507749294);
\draw [line width=2.pt] (5.889838160743649,2.614357899610667)-- (3.8783831835080465,0.590951714017131);
\draw [line width=2.pt] (3.8783831835080465,0.590951714017131)-- (7.00596513318215,0.);
\draw [line width=2.pt] (7.00596513318215,0.)-- (6.561674304718806,0.6523407504001361);
\draw [line width=2.pt] (6.561674304718806,0.6523407504001361)-- (5.889838160743649,2.614357899610667);
\draw [line width=2.pt] (8.164727588954452,4.914752704649513)-- (9.138001229812886,1.9867132398772671);
\draw [line width=2.pt] (9.138001229812886,1.9867132398772671)-- (9.164727588954452,3.3747527046495147);
\draw [line width=2.pt] (9.164727588954452,3.3747527046495147)-- (9.384727588954451,4.114752704649513);
\draw [line width=2.pt] (9.384727588954451,4.114752704649513)-- (8.164727588954452,4.914752704649513);
\draw [line width=2.pt] (10.30437654474037,2.3067721863002317)-- (10.270428515379974,0.5075266301992193);
\draw [line width=2.pt] (10.270428515379974,0.5075266301992193)-- (10.762674941105722,0.20199436595565012);
\draw [line width=2.pt] (10.762674941105722,0.20199436595565012)-- (10.813596985146315,1.7126816724932943);
\draw [line width=2.pt] (10.813596985146315,1.7126816724932943)-- (10.30437654474037,2.3067721863002317);
\draw [line width=2.pt,color=ffqqqq] (5.325855281258707,2.7887361549320286)-- (9.138001229812886,1.9867132398772671);
\draw [line width=1.2pt,dash pattern=on 1pt off 4pt,domain=5.2:10.30437654474037] plot(\x,{(--0.6074457651932192-0.3200589464229646*\x)/-1.1663753149274836});
\draw [line width=1.2pt,dash pattern=on 1pt off 4pt,domain=5.2:10.30437654474037] plot(\x,{(--16.322311942842337-2.32374620098043*\x)/-3.3043765447403697});
\draw [line width=2.pt,color=ffqqqq] (9.138001229812886,1.9867132398772671)-- (10.30437654474037,2.3067721863002317);
\draw [line width=2.pt,color=ffqqqq] (10.30437654474037,2.3067721863002317)-- (11.577427645755233,1.9333438633358704);
\draw [line width=2.pt,color=qqqqff] (5.325855281258707,2.7887361549320286)-- (10.30437654474037,2.3067721863002317);
\begin{scriptsize}
\draw [fill=ududff] (5.325855281258707,2.7887361549320286) circle (2.5pt);
\draw[color=ududff] (5.255394100280098,2.9574667450120273) node {\large $r$};
\draw [fill=ududff] (7.,4.) circle (2.5pt);
\draw[color=ududff] (7.152653831387064,4.127447589575779) node {\large $A$};
\draw [fill=ududff] (7.,-0.01697401468019827) circle (2.5pt);
\draw[color=ududff] (7.151326914663525,-0.016821908035582527) node {\large $B$};
\draw [fill=ududff] (7.,2.9957564963299506) circle (2.5pt);
\draw[color=ududff] (7.152653831387064,3.1266206020573892) node {\large $a$};
\draw [fill=ududff] (6.994443682101728,2.272824156939835) circle (2.5pt);
\draw[color=ududff] (6.897596129095483,2.069408995523878) node {\large $b$};
\draw [fill=ududff] (9.138001229812886,1.9867132398772671) circle (2.5pt);
\draw[color=ududff] (9.18602415461985,1.7867310997900717) node {\large $r'_2$};
\draw [fill=ududff] (10.30437654474037,2.3067721863002317) circle (2.5pt);
\draw[color=ududff] (10.411062701475181,2.4605772899412774) node {\large $r'_1$};
\draw [fill=ududff] (11.577427645755233,1.9333438633358704) circle (2.5pt);
\draw[color=ududff] (11.626856048988717,2.262360918146988) node {\large $t$};
\end{scriptsize}
  \end{tikzpicture}
  \caption{$(A,B)$ is the edge of mesh, the current search node is $n=(I, r)$ where $I=(a,b)$,
  both $r'_1$ and $r'_2$ have labels on the fence that lies on $(A,B)$,
  red segments indicate the actual path, while the blue segment shows the path after the expansion.
  Before the expansion, $f(n) = g(r) + d_o(r, t)$,
  since $h_p(n, r'_1) + g(r'_1) < h_p(n, r'_2) + g(r'_2)$, after the expansion,
  the estimation becomes $g(r) + d_e(r, r'_1) + d_o(r'_1, t)$ which is less than the previous
  value.}
  \label{naive_hf:case}
\end{figure}
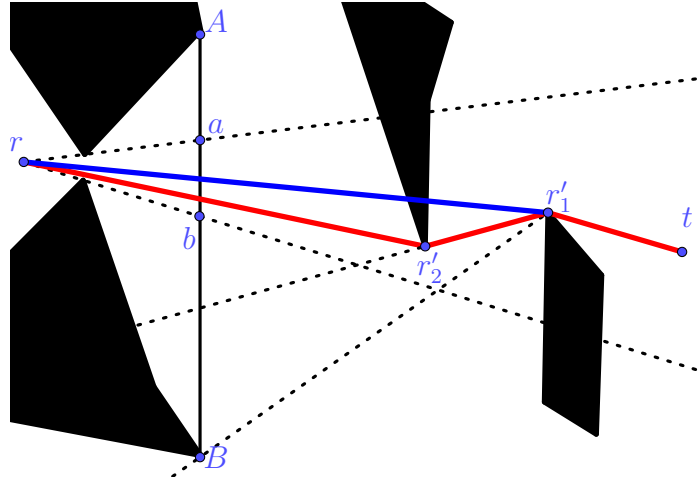

Therefore, when we expand a search node $n=(I,r)$, those labels $n'=(I',r')$ who lay on the fence but
$r,r'$ are not co-visible may ruin the consistency.
To fix this problem, instead of checking visibility in $h'_f$, we can simply keeping the
maximum \textit{f-value} during the search.

\begin{definition}\label{consit_hf:def}
  \textbf{Fence heuristic} $h_f$: for a given search node $n=(I,r)$, $h_f$ guarantees
  $$
    g(r) + h_f(n) = max(f_{parent} , g(r) + h'_f(n))
  $$
  where $f_{parent}$ is the \textit{f-value} of the parent search node.
\end{definition}

\begin{theorem}\label{consit_hf:theorem}
  The $h_f$ is consistent, more specific, $h_f$ is admissible and for a given search node $n$ and it's successor $n'$:
  $$
    f(n) <= f(n')
  $$
\end{theorem}

\begin{proof}\label{consit_hf:proof}
  Let $n=(I, r)$ be the current search node, and $n'=(I',r')$ be any successor.
  First, we show that $h_f$ is admissible by mathematic induction:
  \begin{itemize}
    \item Initially, when $n$ doesn't has any parent, $f(n) = g(r) + h'_f(n)$ which is
      admissible according to theorem~\ref{naive_hf:admis};
    \item Assume $h_f$ for $n$ is admissible during the search;
    \item For successor $n'$, if $f(n') = f(n)$ then $h_f$ is admissible;
    \item Otherwise $f(n') = g(r') + h'_f(n')$ which is also admissible;
  \end{itemize}
  Thus $h_f$ is admissible, and according to definition~\ref{consit_hf:def},
  $f(n') >= f(n)$ is always true, so it is consistent.
\end{proof}
Finally, we implement the fence heuristic search in algorithm~\ref{hf:algo}.
\begin{algorithm}[tb]
  \input{./code/fsearch.pseudo}
  \caption{OkNN search with $h_f$}
  \label{hf:algo}
\end{algorithm}

\subsection{\textit{IER-Polyanya}: Polyanya with Incremental Euclidean Restriction}\label{rep_polyanya}
From previous experiments~\cite{DBLP:conf/socs/ZhaoTH18}, we notice that the \textbf{brute-force Polyanya},
a naive adaption of \textit{Polyanya} that running point-to-point search for each target,
outperforms other sophisticated comptitors under certian scenarios.
Such positive results motivate us to design more efficient algorithm. 

Notice that for same pair of points $(a,b)$, we always have $d_e(q, a) \le d_o(q, b)$,
this is called \textit{Euclidean lower-bound property}~\cite{DBLP:conf/edbt/ZhangPMZ04}.
Utilize such property, we can prune the search space,
similar ideas also appeared in previous literatures~\cite{DBLP:conf/edbt/ZhangPMZ04,DBLP:conf/bncod/XiaHT04,DBLP:journals/pvldb/AbeywickramaCT16}.
We implement this idea as follow:
given query point $q$, the algorithm use \textit{incremental nearest neighbor} query in
\textit{R-tree}~\cite{DBLP:journals/tods/HjaltasonS99} to visit target $t \in T$ in order of $d_e(q, t)$,
then use Polyanya to compute $d_o(q, t)$ and keep track k-nearest neighbours with $d_o$ metric; 
the algorithm terminates when $d_e(q, t) > d_o(q, t_k)$ where
$t_k$ is the current $k$-th obstacle nearest neighbour of $q$, see in algorithm~\ref{repoly}.
\begin{algorithm}[tb]
  \input{./code/repetitive_polyanya.pseudo}
  \caption{Repeated Polyanya}
  \label{repoly}
\end{algorithm}

\begin{lemma}
  When algorithm~\ref{repoly} terminates, the \textit{candidates} stores top-$k$ smallest
  obstacle distances.
\end{lemma}

\begin{proof}
  Assume there is an unvisited $t'$ that $d_o(q, t') < d_o(q, t_k)$,
  where $t_k$ is the $k$-th nearest neighbor. Since $t'$ is s unvisited,
  we have $d_e(q, t') >= d_o(q, t_k)$; according to \textit{Euclidean lower-bound
  property}, we have $d_e(q, t') \le d_o(q, t')$ and $d_e(q, t_k) \le d_o(q, t_k)$,
  therefore we have $d_o(q, t') >= d_o(q, t_k)$ which contradict to the assumption,
  so such $t'$ doesn't exist.
\end{proof}
\noindent
The effectiveness of \textit{IER-Polyanya} affected by two factors:
\begin{itemize}
  \item \textbf{False hit}: a searched target not being retrieved is called false hit.
    False hit is more likely happen when Euclidean metric is misleading,
    for example in Fig~\ref{obs_dis}, $D$ is the first target to search but wouldn't
    be retrieved when $k<=3$.
  \item \textbf{Overlapping search space}: since each search is independent, same search space
    may be explored multiple times.
\end{itemize}
\noindent
In section~\ref{exp}, we will examine the performance of \textit{IER-Polyanya} based these considerations.

\section{Empirical Analysis} \label{exp}
OkNN problem often appears in both AI pathfinding and spatial query preprocessing,
so we examine the performance of proposed methods on different maps that corresponding two
these two application scenarios. Fig~\ref{maps} shows the set of maps that we
consider in our experimental evaluation.
\begin{figure}[tb]
  \centering
  \begin{subfigure}[tb]{\textwidth}
    \centering
    \begin{subfigure}[tb]{0.40\textwidth}
      \includegraphics[width=\textwidth]{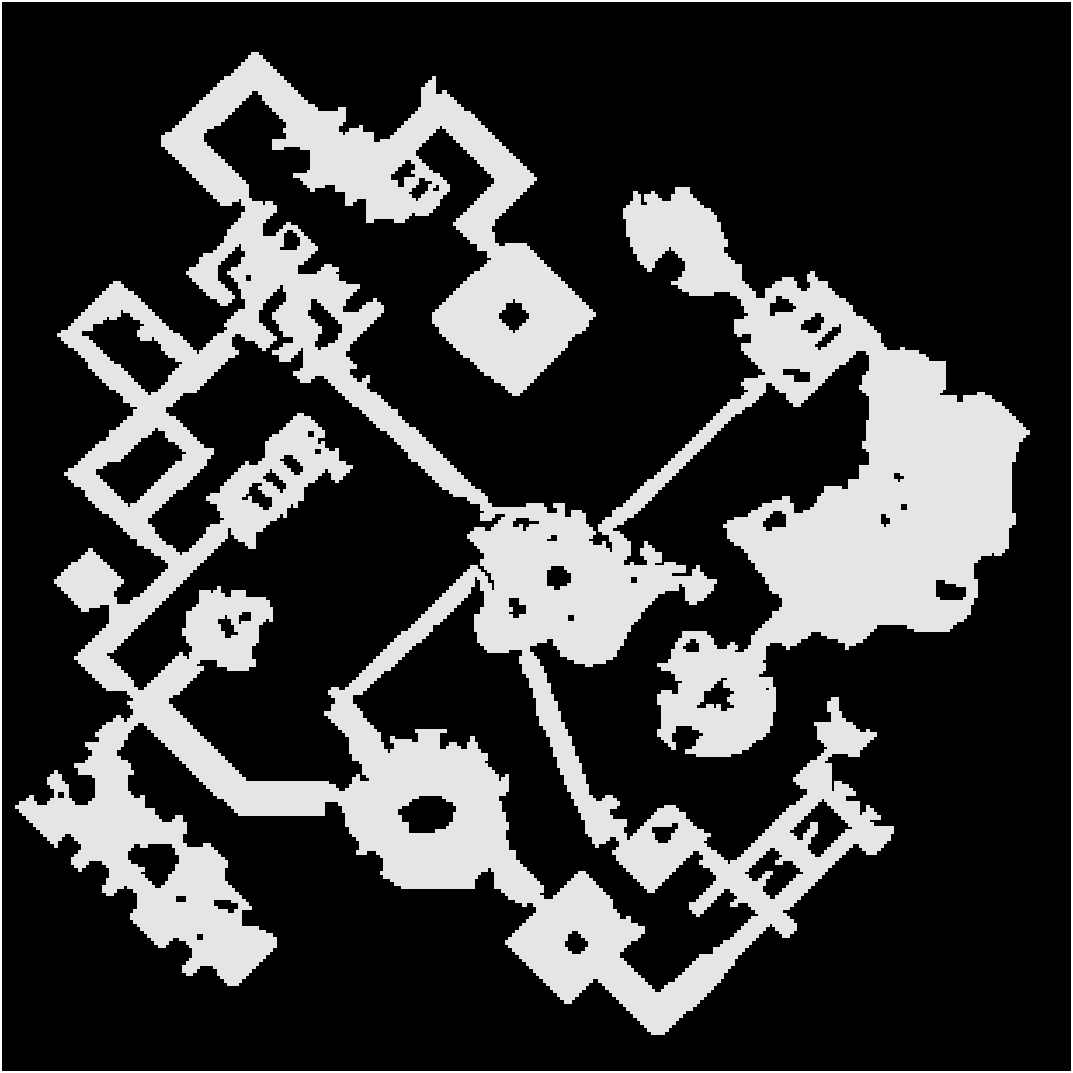}
      \caption{AR0602SR.  This map has 47 obstacles and 5504 vertices.
      The navigation mesh mesh comprises 5594 polygons.}
    \end{subfigure}\hfill
     \vspace{1em}
    \begin{subfigure}[tb]{0.40\textwidth}
      \includegraphics[width=\textwidth]{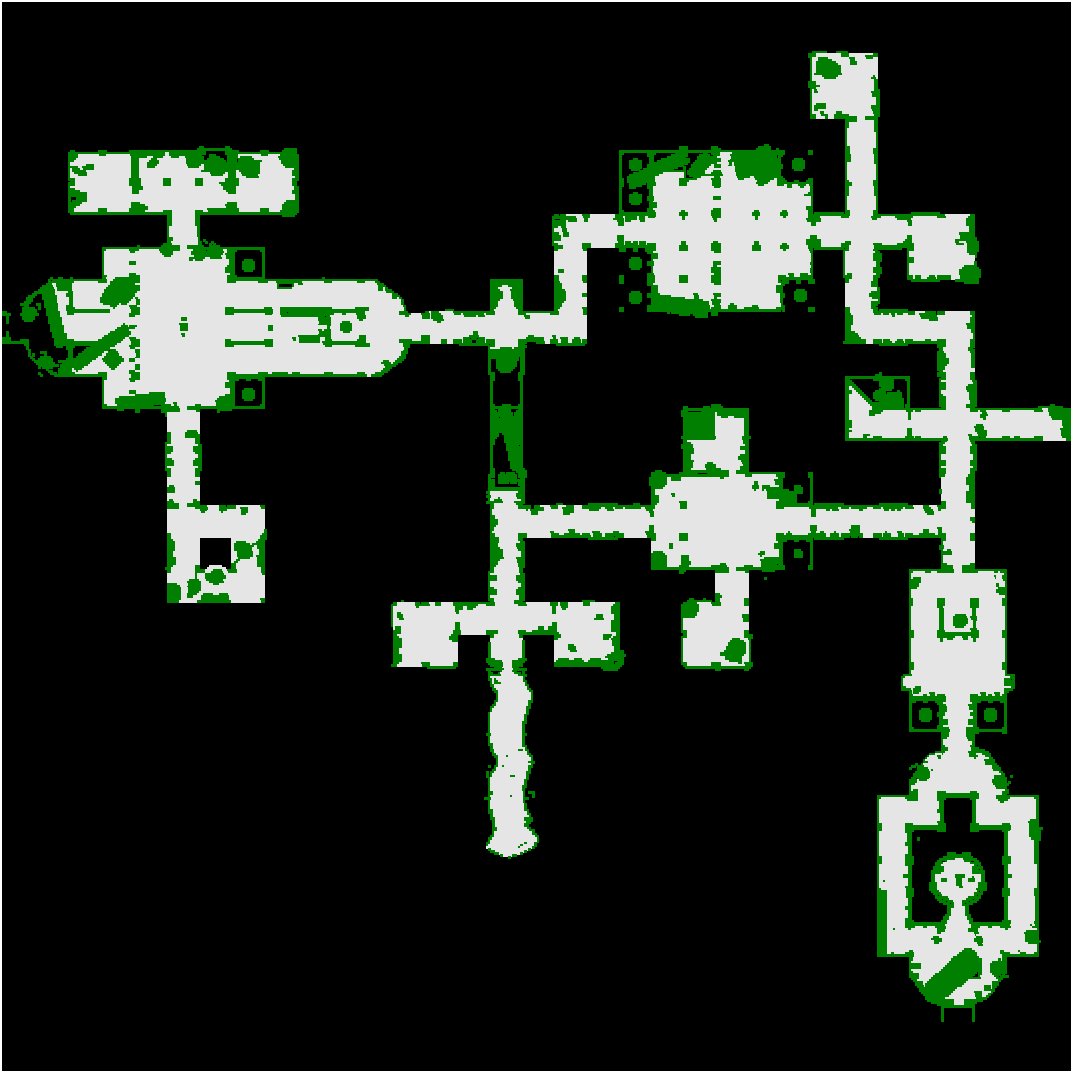}
      \caption{brc202d. This map has 75 obstacles and 4035 vertices.
      The navigation mesh comprises 4164 polygons.}
    \end{subfigure}
  \end{subfigure}
  \vspace{1em}
  \begin{subfigure}[tb]{\textwidth}
    \centering
    \begin{subfigure}[tb]{0.4\textwidth}
      \includegraphics[width=\textwidth]{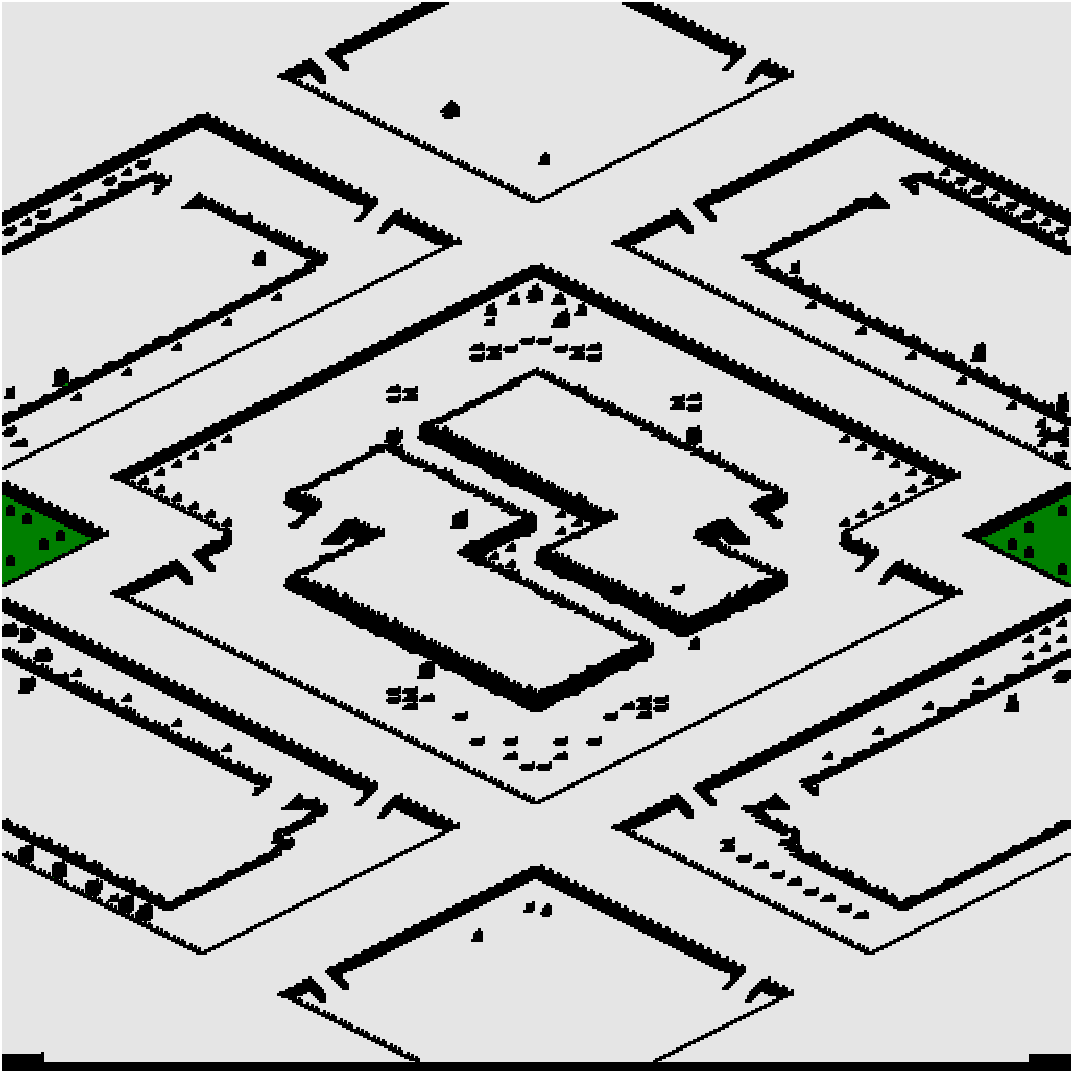}
      \caption{CatwalkAlley. This map has 157 obstacles and 15301 vertices.
      The navigation mesh comprises 15482 polygons. }
    \end{subfigure}\hfill
    \begin{subfigure}[tb]{0.4\textwidth}
      \includegraphics[width=\textwidth]{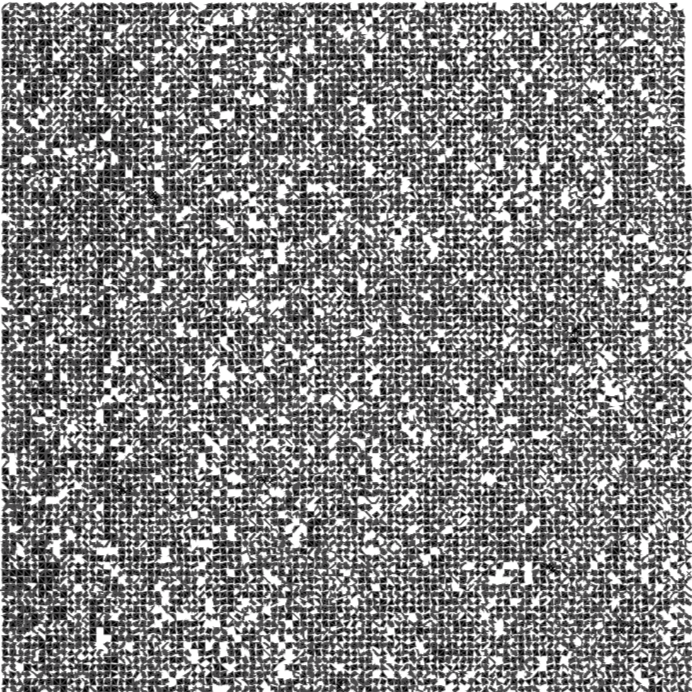}
      \caption{8K tiled obstacles. This map has 8385 obstacles and
      109171 vertices. The navigation mesh comprises 125937 polygons.}
    \end{subfigure}
  \end{subfigure}
  \caption{Maps and summaries. NB: black and green indicates non-traversable areas.}
  \label{maps}
\end{figure}


We choose three maps from a well known AI pathfinding benchmark
sets~\cite{DBLP:journals/tciaig/Sturtevant12}: \textit{AR0602SR, brc202d, CatwalkAlley};
and a synthetic map \textit{8K tiled obstacles} from our previous
work~\cite{DBLP:conf/socs/ZhaoTH18}\footnote{There are few real datasets (e.g. http://www.rtreeportal.org/), but all expired that not publicly
available},
the way we created such synthetic map is described below: 
\begin{itemize}
  \item \textbf{Tiled obstacles map}. we extract the shape of all parks in Australia from \textit{OpenStreetMap}\cite{OpenStreetMap} and
    use these shapes as polygonal obstacles. There are initially 9000 such polygons\footnote{In
    our previous~\cite{DBLP:conf/socs/ZhaoTH18}, it's called 9000 obstacles map.}, 
    after removing invalid polygons, there are 8385 polygons left.
  Next we generate a map by tiling all obstacles in the empty square plane.
  For the tiling, we first divide the square plane into grid having $\lceil\sqrt{|O|}\rceil$ number of rows and columns.
  Then we assign each polygon to a single grid cell and normalize the shape of polygon by to fit inside the cell.
\end{itemize}

The navigation mesh of map is generated by \textit{Constrained Delaunay Triangulation}, which
is $O(nlogn)$; the implementation of such algorithm is in library \textit{Fade2D}
\footnote{http://www.geom.at/fade2d/html}, the total time on such preprocessing is about 6s.

\begin{table}[tb]
  \centering
  \begin{tabular}{l|l}
  \toprule
  Parameters          &                     Values                                  \\
  \midrule
    $k$               &   1, \textbf{5}, 10, 25, 50                                 \\
  \midrule
  Maps                &  AR0602SR, brc202d, CatwalkAlley, \textbf{8K-tiled-obs}     \\
  \midrule
  Target Density ($d$)&  0.0001, 0.001, \textbf{0.01}, 0.1                        \\
  \midrule
  Target distribution &     \textbf{random}, clustered                             \\
  \bottomrule
  \end{tabular}
  \caption{Parameters (default in bold)}
  \label{para}
\end{table}

In each query processing experiment, we choose parameters from Table~\ref{para}.
Notice that we define \textit{target density} by $|T|/|V|$ where $|V|$ is the number of vertices, 
while other works~\cite{DBLP:conf/edbt/ZhangPMZ04,DBLP:conf/sigmod/GaoZ09,DBLP:journals/isci/GaoLMY16} used $|T|/|O|$ as
their density where $|O|$ is the number of obstacles. The reason is that game maps usually have
few large continuous obstacles.

We're using $1000$ random query points for each setting, grouping
results by \textit{x-axis}, and computing average; the size of each bucket is at least 10; 
All algorithms appear in experiments are:
\begin{itemize}
  \item LVG: Local Visibility Graph~\cite{DBLP:conf/edbt/ZhangPMZ04};
  \item h$_v$: Interval heuristic~\cite{DBLP:conf/socs/ZhaoTH18};
  \item h$_t$: Target heuristic~\cite{DBLP:conf/socs/ZhaoTH18};
  \item h$_f$: Fence heuristic (Algorithm~\ref{hf:algo});
  \item fc: Fence Checking (Algorithm~\ref{fnn});
  \item IER-Poly: Polyanya with Incremental Euclidean Restriction (Algorithm~\ref{repoly});
\end{itemize}
\noindent
These algorithms are implemented in C++ and compiled with \textit{clang-902.0.39.1} using \textit{-O3} flag,
under \textit{x86\_64-apple-darwin17.5.0} platform.
All of our source code and test data set are  publicly available \footnote{http://bitbucket.org/dharabor/pathfinding}.
All experiments are performed on a 2.5 GHz Intel Core i7 machine with 16GB of RAM and running OSX 10.13.4. 

\subsection{Experiment 1: preprocessing}
This experiment is to examine the performance of preprocessing.
For each map, we generated random distributed target set with four density and run floodfill
with fence pruning (algorithm~\ref{floodfill}).

\textbf{Time cost}. We collect the execution time of each instance and group by map, each
group contains data from different density. Tabel~\ref{exp1:fcost} shows that the \textit{std}
of each group is relative small but the mean of groups are very different,
so we can conclude that the time cost of preprocessing is not sensitive to density of targets
but the map itself (e.g. number of vertices).

\begin{table}[tb]
\centering
\begin{tabular}{lrrrrrrrr}
\toprule
  map             &  vertices   &    mean &    std &     min &     25\% &     50\% &     75\% &     max \\
\midrule
AR0602SR          &  5504       &    6.92 &   0.14 &    6.72 &    6.84 &    6.92 &    6.98 &    7.21 \\
brc202d           &  4164       &    6.11 &   0.25 &    5.63 &    5.98 &    6.09 &    6.25 &    6.54 \\
CatwalkAlley      &  15301      &   26.73 &   2.35 &   22.83 &   25.55 &   26.15 &   28.49 &   30.87 \\
8K-tiled-obs      &  109171     &  537.16 &  24.94 &  499.93 &  522.24 &  537.60 &  552.43 &  580.43 \\
\bottomrule
\end{tabular}
\caption{Processing time (ms)}
\label{exp1:fcost}
\end{table}

\textbf{Label size}. This metric indicates the space cost and the performance during the query processing.
We collect the number of labels on each mesh edge and group by map, then
in each group, we further group data by density\footnote{There are $\approx 1\%$ of edges
have more than 10 and up to 300 labels, we filter these out for better visualisation}.
Fig~\ref{exp1:lsize} shows that for all maps and all densities, medians of label size per edge
are simlar.

\begin{figure}[bt]
  \begin{subfigure}[tb]{.25\textwidth}
    \includegraphics[width=\textwidth]{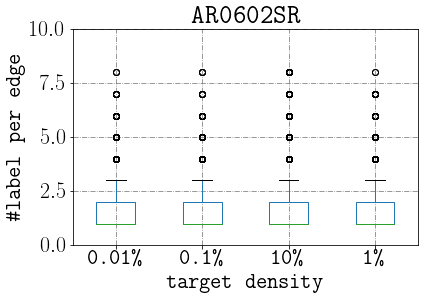}
  \end{subfigure}%
  \begin{subfigure}[tb]{.25\textwidth}
    \includegraphics[width=\textwidth]{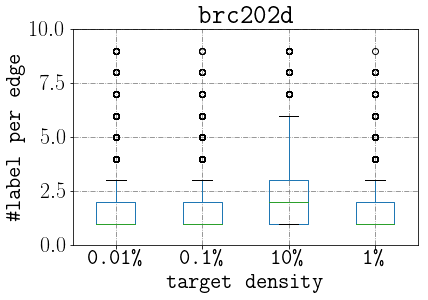}
  \end{subfigure}%
  \begin{subfigure}[tb]{.25\textwidth}
    \includegraphics[width=\textwidth]{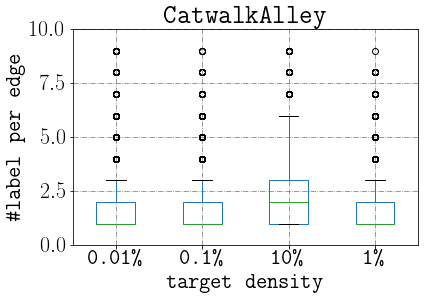}
  \end{subfigure}%
  \begin{subfigure}[tb]{.25\textwidth}
    \includegraphics[width=\textwidth]{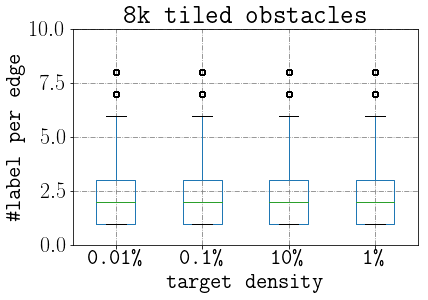}
  \end{subfigure}
  \caption{Number of labels per edge}
  \label{exp1:lsize}
\end{figure}

\subsection{Experiment 2: random distributed targets}
The aim of this experiment is to examine the performance of proposed algorithms when $k$ and
$d$ (density) change, by default, $k=5$ and $d=0.01$. The effectiveness of
heuristic function is measured by total search time divide by number of generated node, so for
same search time, larger heuristic cost meaning smaller search space.
We also include the previous state-of-the-art \textit{LVG} in comparison, and since this method
may reach quadratic time complexity under some scenarios, we clip the execution time by $10^3$
ms.

From Fig~\ref{exp2:uni} and Table~\ref{exp2:hcost}, we make the following observations:
\begin{itemize}
  \item \textbf{IER-Polyanya} always outperforms others, and when $k$ is fixed it's nearly order of
    magnitude faster than $h_f$ and $h_v$. The reason is that Euclidean metric in
    this map with random distributed targets is very good, so it has small number of false hit
    and thus get benefit from the fast point-to-point search.

  \item \textbf{Fence heuristic} has simlar number of generated search nodes to $h_t$ regarding
    the heuristic cost, and in general, it also has similar time cost to $h_v$. 
    Besides, it significantly outperforms $h_v$ when $k=1$ and $d$ is fixed.
    So we can conclude that $h_f$ has both good time cost and small number of generated nodes.

  \item \textbf{LVG} is completely dominated by others in all cases, so we can remove this
    competitor in further experiments.
\end{itemize}

\begin{figure}[bt]
  \centering
  \begin{subfigure}[b]{\textwidth}
    \centering
    \begin{subfigure}[b]{0.45\textwidth}
      \includegraphics[width=\textwidth]{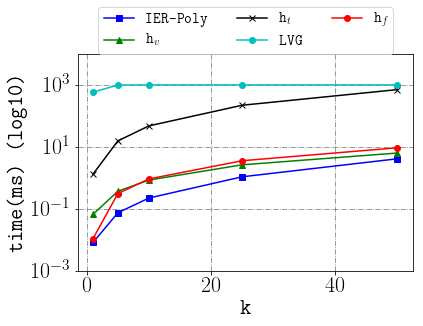}
      \caption{varying $k$ ($d=0.01$)}
    \end{subfigure}%
    \begin{subfigure}[b]{0.45\textwidth}
    \centering
      \includegraphics[width=\textwidth]{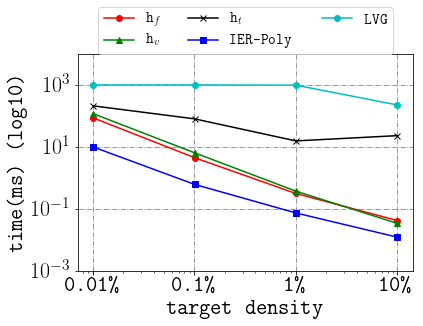}
      \caption{varying $d$ ($k=5$)}
    \end{subfigure}
  \end{subfigure}
  \caption{Performance on the default map (8K tiled obstacles)}
  \label{exp2:uni}
\end{figure}

\begin{table}[bth]
  \centering
  \begin{tabular}{lrrrrrrr}
  \toprule
  heuristic &    mean &     std &   min &    25\% &    50\% &     75\% &      max \\
  \midrule
  $h_f$ &    0.82 &    0.19 &  0.05 &   0.72 &   0.80 &    0.91 &     2.83 \\
  $h_p$ &    0.26 &    0.07 &  0.02 &   0.23 &   0.25 &    0.29 &     1.77 \\
  $h_t$ &  162.19 &  262.96 &  9.23 &  71.29 &  93.23 &  112.16 &  8191.00 \\
  $h_v$ &    0.42 &    0.12 &  0.08 &   0.34 &   0.39 &    0.44 &     1.42 \\
  \bottomrule
  \end{tabular}
  \caption{The average cost ($\mu$s) of each heuristic function}
  \label{exp2:hcost}
\end{table}

\subsection{Experiment 3: clustered targets}
One drawback of \textit{IER-Polyanya} is that each repeated search is independent, so that it
makes many redundant computation when those searches are overlapped.
The aim of this experiment is to examine the performance of \textit{IER-Polyanya} when search
space overlapped. We create a single cluster of targets, and to get rid of the influence of false hit,
we let the size of cluster be $50$, so that when $k=50$ the false hit becames $0$.

From Figs~\ref{exp3:fhit},\ref{exp3:varyk} we make following observations:
\begin{itemize}
  \item when $k<50$, \textit{IER-Polyanya} is outpeformed by $h_f$ because the false hit;
  \item when $k=50$, false hit is $0$, but \textit{IER-Polyanya} is still significantly outpeformed by $h_f$,
    while in previous experiment, it's faster than $h_f$ for all $k$. So we can conclude that
    \textit{IER-Polyanya} has bad performance when search space overlapped.
\end{itemize}
\begin{figure}[bt]
  \centering
  \begin{subfigure}[b]{\textwidth}
    \centering
    \begin{subfigure}[b]{0.45\textwidth}
      \centering
      \includegraphics[width=\textwidth]{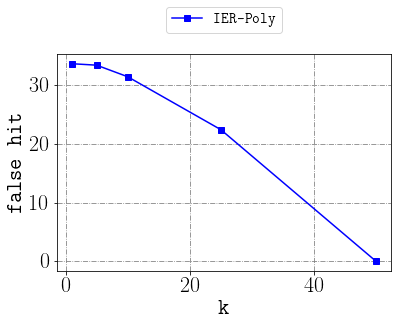}
      \caption{}
      \label{exp3:fhit}
    \end{subfigure}%
    \begin{subfigure}[b]{0.45\textwidth}
      \centering
      \includegraphics[width=\textwidth]{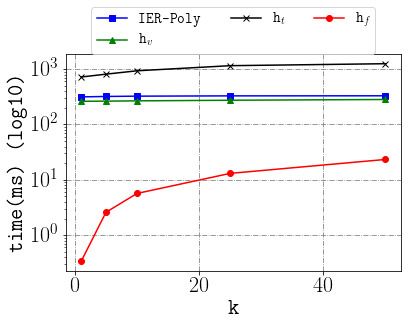}
      \caption{}
      \label{exp3:varyk}
    \end{subfigure}
  \end{subfigure}
  \caption{Performance on clustered targets}
  \label{exp3}
\end{figure}

\subsection{Experiment 4: varying maps}
Another drawback of \textit{IER-Polyanya} is that the Euclidean metric can be misleading in some
maps. The aim of this experiment is to examine the performance of \textit{IER-Polyanya} on different
types of maps.

We run $1000$ random queries with default parameter $k=5, d=0.01$ where targets are
randomly distributed. To show the influence of misleadingness, we analyze results as follow:
\begin{itemize}
  \item We use the number of nodes generated by $h_v$ measure the difficulty of search;
  \item For each search, we evaluate the speed-up factor of each method by time cost of such
    method divide by time cost of $h_v$;
  \item We plot the number of false-hit and the speed-up factor with increasing difficulty for
    each map.
\end{itemize}
Such speed-up comparison also appears in other pathfinding literature~\cite{DBLP:journals/jair/HaraborGOA16}.

From fig~\ref{exp4:maps} we make following observations:
\begin{itemize}
  \item In all maps, when false hit increases the speed-up
    factor decreases, meaning that false hit affect the performance of \textit{IER-Polyanya}.

  \item \textit{AR0602SR,brc202d,CatwalkAlley} have more false hit than \textit{8K tiled
    obstacles}, meaning that Euclidean metric is misleading in those maps.

  \item In those misleading maps, speed-up factor of \textit{IER-Polyanya} is relative small, and
    sometimes less than $1$ (worse than $h_v$),
    meaning that \textit{IER-Polyanya} doesn't work well when Euclidean metric becomes
    misleading.
\end{itemize}

\subsection{Experiment 5: Nearest neighbor query}
\begin{figure}[bt]
  \centering
  \begin{subfigure}[bt]{.5\textwidth}
  \includegraphics[width=\textwidth]{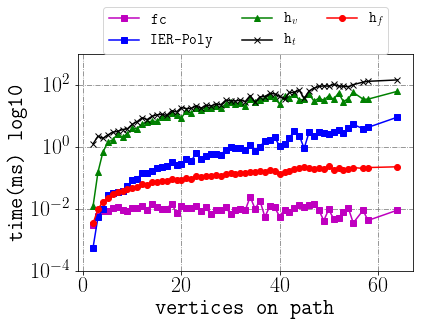}
  \caption{}
  \label{exp5:vary_vnum}
  \end{subfigure}%
  \begin{subfigure}[bt]{.5\textwidth}
  \includegraphics[width=\textwidth]{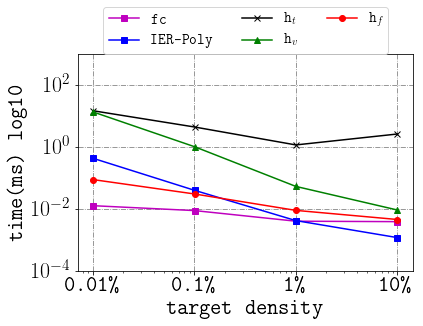}
  \caption{}
  \label{exp5:vary_t}
  \end{subfigure}
  \caption{Performance on nn query}
  \label{exp5}
\end{figure}
The advantage of \textit{fence checking} is only searching the last vertex of the shortest path,
the aim of this experiment is to examine when it performs well.
In the experiment, we run $1000$ random queries with different densities,
then we analyze results in two aspects:
(i) grouping results by number of vertices on path to examine the performance when it gets
advantage;
(ii) grouping results by $d$ to examine the general performance in different density.

Fig~\ref{exp5:vary_vnum} shows that \textit{fence checking} is not sensitive to the number of
vertices on path and can be orders of magnitude faster than other competitors when it gets
advantage.

Fig~\ref{exp5:vary_t} shows that \textit{fence checking} outperforms others when $d<0.01$, the
reason is that when density increases, the number of vertices on shortest path decrease,
thus \textit{fence checking} gradually lose the advantage.

\begin{figure}[tb]
  \begin{subfigure}[b]{\textwidth}
    \begin{subfigure}[b]{0.45\textwidth}
      \centering
      \includegraphics[width=.9\textwidth]{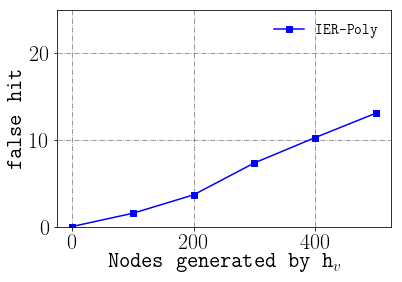}
    \end{subfigure}%
    \begin{subfigure}[b]{0.45\textwidth}
      \centering
      \includegraphics[width=.9\textwidth]{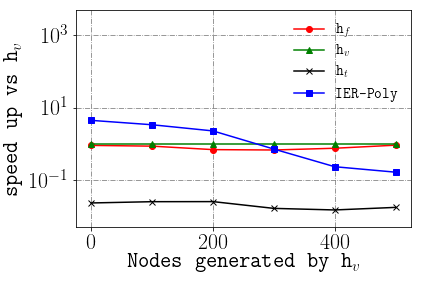}
    \end{subfigure}
    \caption{AR0602SR}
  \end{subfigure}
  \hfill
  \begin{subfigure}[b]{\textwidth}
    \begin{subfigure}[b]{0.45\textwidth}
      \centering
      \includegraphics[width=.9\textwidth]{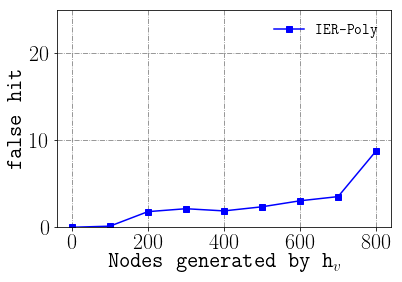}
    \end{subfigure}%
    \begin{subfigure}[b]{0.45\textwidth}
      \centering
      \includegraphics[width=.9\textwidth]{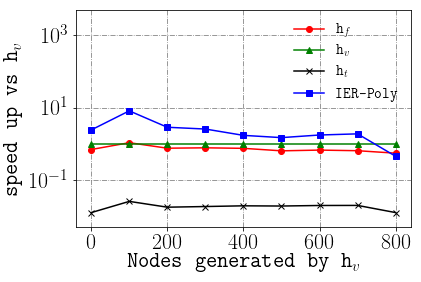}
    \end{subfigure}
    \caption{brc202d}
  \end{subfigure}
  \hfill
  \begin{subfigure}[b]{\textwidth}
    \begin{subfigure}[b]{0.45\textwidth}
      \centering
      \includegraphics[width=.9\textwidth]{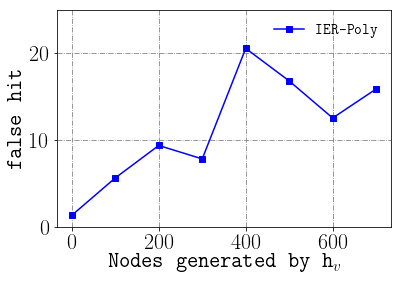}
    \end{subfigure}%
    \begin{subfigure}[b]{0.45\textwidth}
      \centering
      \includegraphics[width=.9\textwidth]{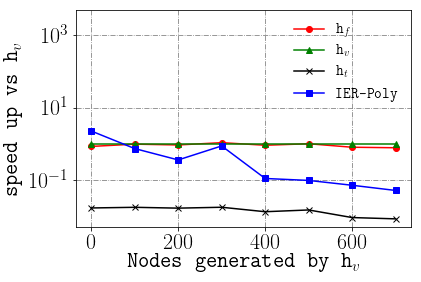}
    \end{subfigure}
    \caption{CatwalkAlley}
  \end{subfigure}
  \hfill
  \begin{subfigure}[b]{\textwidth}
    \begin{subfigure}[b]{0.45\textwidth}
      \centering
      \includegraphics[width=.9\textwidth]{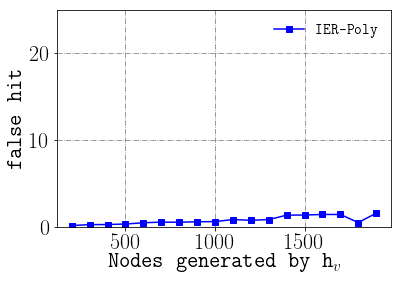}
    \end{subfigure}%
    \begin{subfigure}[b]{0.45\textwidth}
      \centering
      \includegraphics[width=.9\textwidth]{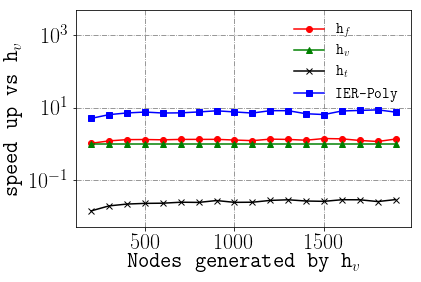}
    \end{subfigure}
    \caption{8K tiled obstacles}
  \end{subfigure}
  \caption{Performance on different maps}
  \label{exp4:maps}
\end{figure}

\section{Conclusion and Future Work} \label{conc}
We study the obstacle k-nearest neighbours problem which appears in both AI
pathfinding and spatial database areas. In this paper, we proposed two efficient methods:
\textbf{fence heuristic} and \textbf{IER-Polyanya}.

\textbf{Fence heuristic} utilizes precomputed
information to guide the search process, compare to $h_v$ and $h_t$ in our previous work, it
has both cheap heuristic cost and small search space, and it performs well in whatever target
density. The preprocessing for \textbf{Fence heuristic} has bad theoretical upper bound,
but it is efficient in practical when a pruning strategy be applied.
\textbf{IER-Polyanya} utilizes a fast point-to-point search to compute 
obstacle distance and Euclidean restriction to prune candidate set,
it sometimes outperforms \textbf{Fence heuristic} order of magnitude, but such performance
is not stable that sensitive to the distribution of targets and obstacles.

There are several possible directions for future work. Firstly, current \textbf{Fence
heuristic} performs a linear search for labels on fence to compute the heuristic value,
an obvious improvement is to store labels by their interval,
and during the search, only retrieval labels with overlapped interval (meaning visible).
Secondly, notice that adding or deleting targets only affect a small area of the map, so that
we can develop these operations to efficiently maintain fence labels when targets change.
Thirdly, one possible way to deal the worst case in preprocessing is reconstructing part of the
navigation mesh.

\clearpage
\bibliographystyle{theapa}
\bibliography{ref}

\begin{thebibliography}{}

\bibitem[\protect\BCAY{Abeywickrama, Cheema,\ \BBA\ Taniar}{Abeywickrama
  et~al.}{2016}]{DBLP:journals/pvldb/AbeywickramaCT16}
Abeywickrama, T., Cheema, M.~A., \BBA\ Taniar, D. \BBOP2016\BBCP.
\newblock \BBOQ k-nearest neighbors on road networks: {A} journey in
  experimentation and in-memory implementation\BBCQ\
\newblock {\Bem {PVLDB}}, {\Bem 9\/}(6), 492--503.

\bibitem[\protect\BCAY{Aha, Molineaux,\ \BBA\ Ponsen}{Aha
  et~al.}{2005}]{amp-ltw-05}
Aha, D.~W., Molineaux, M., \BBA\ Ponsen, M. \BBOP2005\BBCP.
\newblock \BBOQ Learning to win: Case-based plan selection in a real-time
  strategy game\BBCQ\
\newblock In Mu{\~{n}}oz-{\'A}vila, H.\BBACOMMA\  \BBA\ Ricci, F.\BEDS, {\Bem
  Proceedings of the International Conference on Case-Based Reasoning (ICCBR)},
  \BPGS\ 5--20.

\bibitem[\protect\BCAY{Arkin}{Arkin}{1989}]{DBLP:journals/robotica/Arkin89}
Arkin, R.~C. \BBOP1989\BBCP.
\newblock \BBOQ Navigational path planning for a vision-based mobile
  robot\BBCQ\
\newblock {\Bem Robotica}, {\Bem 7\/}(1), 49--63.

\bibitem[\protect\BCAY{Beckmann, Kriegel, Schneider,\ \BBA\ Seeger}{Beckmann
  et~al.}{1990}]{DBLP:conf/sigmod/BeckmannKSS90}
Beckmann, N., Kriegel, H., Schneider, R., \BBA\ Seeger, B. \BBOP1990\BBCP.
\newblock \BBOQ The r*-tree: An efficient and robust access method for points
  and rectangles\BBCQ\
\newblock In Garcia{-}Molina, H.\BBACOMMA\  \BBA\ Jagadish, H.~V.\BEDS, {\Bem
  Proceedings of the 1990 {ACM} {SIGMOD} International Conference on Management
  of Data, Atlantic City, NJ, USA, May 23-25, 1990.}, \BPGS\ 322--331. {ACM}
  Press.

\bibitem[\protect\BCAY{Cheung\ \BBA\ Fu}{Cheung\ \BBA\
  Fu}{1998}]{DBLP:journals/sigmod/CheungF98}
Cheung, K.~L.\BBACOMMA\  \BBA\ Fu, A.~W. \BBOP1998\BBCP.
\newblock \BBOQ Enhanced nearest neighbour search on the r-tree\BBCQ\
\newblock {\Bem {SIGMOD} Record}, {\Bem 27\/}(3), 16--21.

\bibitem[\protect\BCAY{Chew}{Chew}{1989}]{DBLP:journals/algorithmica/Chew89}
Chew, L.~P. \BBOP1989\BBCP.
\newblock \BBOQ Constrained delaunay triangulations\BBCQ\
\newblock {\Bem Algorithmica}, {\Bem 4\/}(1), 97--108.

\bibitem[\protect\BCAY{Cui, Harabor,\ \BBA\ Grastien}{Cui
  et~al.}{2017}]{DBLP:conf/ijcai/CuiHG17}
Cui, M., Harabor, D.~D., \BBA\ Grastien, A. \BBOP2017\BBCP.
\newblock \BBOQ Compromise-free pathfinding on a navigation mesh\BBCQ\
\newblock In Sierra, C.\BED, {\Bem Proceedings of the Twenty-Sixth
  International Joint Conference on Artificial Intelligence, {IJCAI} 2017,
  Melbourne, Australia, August 19-25, 2017}, \BPGS\ 496--502. ijcai.org.

\bibitem[\protect\BCAY{de~Berg}{de~Berg}{2000}]{DBLP:books/lib/Berg00}
de~Berg, M. \BBOP2000\BBCP.
\newblock {\Bem Computational geometry: algorithms and applications, 2nd
  Edition}.
\newblock Springer.

\bibitem[\protect\BCAY{Demyen\ \BBA\ Buro}{Demyen\ \BBA\
  Buro}{2006}]{DBLP:conf/aaai/DemyenB06}
Demyen, D.\BBACOMMA\  \BBA\ Buro, M. \BBOP2006\BBCP.
\newblock \BBOQ Efficient triangulation-based pathfinding\BBCQ\
\newblock In {\Bem Proceedings, The Twenty-First National Conference on
  Artificial Intelligence and the Eighteenth Innovative Applications of
  Artificial Intelligence Conference, July 16-20, 2006, Boston, Massachusetts,
  {USA}}, \BPGS\ 942--947. {AAAI} Press.

\bibitem[\protect\BCAY{Gao, Liu, Miao,\ \BBA\ Yang}{Gao
  et~al.}{2016}]{DBLP:journals/isci/GaoLMY16}
Gao, Y., Liu, Q., Miao, X., \BBA\ Yang, J. \BBOP2016\BBCP.
\newblock \BBOQ Reverse k-nearest neighbor search in the presence of
  obstacles\BBCQ\
\newblock {\Bem Inf. Sci.}, {\Bem 330}, 274--292.

\bibitem[\protect\BCAY{Gao, Yang, Chen, Zheng,\ \BBA\ Chen}{Gao
  et~al.}{2011}]{DBLP:conf/gis/GaoYCZC11}
Gao, Y., Yang, J., Chen, G., Zheng, B., \BBA\ Chen, C. \BBOP2011\BBCP.
\newblock \BBOQ On efficient obstructed reverse nearest neighbor query
  processing\BBCQ\
\newblock In Cruz, I.~F., Agrawal, D., Jensen, C.~S., Ofek, E., \BBA\ Tanin,
  E.\BEDS, {\Bem 19th {ACM} {SIGSPATIAL} International Symposium on Advances in
  Geographic Information Systems, {ACM-GIS} 2011, November 1-4, 2011, Chicago,
  IL, USA, Proceedings}, \BPGS\ 191--200. {ACM}.

\bibitem[\protect\BCAY{Gao\ \BBA\ Zheng}{Gao\ \BBA\
  Zheng}{2009}]{DBLP:conf/sigmod/GaoZ09}
Gao, Y.\BBACOMMA\  \BBA\ Zheng, B. \BBOP2009\BBCP.
\newblock \BBOQ Continuous obstructed nearest neighbor queries in spatial
  databases\BBCQ\
\newblock In {\c{C}}etintemel, U., Zdonik, S.~B., Kossmann, D., \BBA\ Tatbul,
  N.\BEDS, {\Bem Proceedings of the {ACM} {SIGMOD} International Conference on
  Management of Data, {SIGMOD} 2009, Providence, Rhode Island, USA, June 29 -
  July 2, 2009}, \BPGS\ 577--590. {ACM}.

\bibitem[\protect\BCAY{Ghosh\ \BBA\ Mount}{Ghosh\ \BBA\
  Mount}{1991}]{DBLP:journals/siamcomp/GhoshM91}
Ghosh, S.~K.\BBACOMMA\  \BBA\ Mount, D.~M. \BBOP1991\BBCP.
\newblock \BBOQ An output-sensitive algorithm for computing visibility
  graphs\BBCQ\
\newblock {\Bem {SIAM} J. Comput.}, {\Bem 20\/}(5), 888--910.

\bibitem[\protect\BCAY{Guttman}{Guttman}{1984}]{DBLP:conf/sigmod/Guttman84}
Guttman, A. \BBOP1984\BBCP.
\newblock \BBOQ R-trees: {A} dynamic index structure for spatial
  searching\BBCQ\
\newblock In Yormark, B.\BED, {\Bem SIGMOD'84, Proceedings of Annual Meeting,
  Boston, Massachusetts, USA, June 18-21, 1984}, \BPGS\ 47--57. {ACM} Press.

\bibitem[\protect\BCAY{Harabor, Grastien, {\"{O}}z,\ \BBA\ Aksakalli}{Harabor
  et~al.}{2016}]{DBLP:journals/jair/HaraborGOA16}
Harabor, D.~D., Grastien, A., {\"{O}}z, D., \BBA\ Aksakalli, V. \BBOP2016\BBCP.
\newblock \BBOQ Optimal any-angle pathfinding in practice\BBCQ\
\newblock {\Bem J. Artif. Intell. Res.}, {\Bem 56}, 89--118.

\bibitem[\protect\BCAY{Hjaltason\ \BBA\ Samet}{Hjaltason\ \BBA\
  Samet}{1999}]{DBLP:journals/tods/HjaltasonS99}
Hjaltason, G.~R.\BBACOMMA\  \BBA\ Samet, H. \BBOP1999\BBCP.
\newblock \BBOQ Distance browsing in spatial databases\BBCQ\
\newblock {\Bem {ACM} Trans. Database Syst.}, {\Bem 24\/}(2), 265--318.

\bibitem[\protect\BCAY{Kallmann}{Kallmann}{2005}]{kallmann2005path}
Kallmann, M. \BBOP2005\BBCP.
\newblock \BBOQ Path planning in triangulations\BBCQ\
\newblock In {\Bem Proceedings of the IJCAI workshop on reasoning,
  representation, and learning in computer games}, \BPGS\ 49--54.

\bibitem[\protect\BCAY{Kamel\ \BBA\ Faloutsos}{Kamel\ \BBA\
  Faloutsos}{1994}]{DBLP:conf/vldb/KamelF94}
Kamel, I.\BBACOMMA\  \BBA\ Faloutsos, C. \BBOP1994\BBCP.
\newblock \BBOQ Hilbert r-tree: An improved r-tree using fractals\BBCQ\
\newblock In Bocca, J.~B., Jarke, M., \BBA\ Zaniolo, C.\BEDS, {\Bem VLDB'94,
  Proceedings of 20th International Conference on Very Large Data Bases,
  September 12-15, 1994, Santiago de Chile, Chile}, \BPGS\ 500--509. Morgan
  Kaufmann.

\bibitem[\protect\BCAY{Ooi}{Ooi}{1987}]{DBLP:conf/btw/Ooi87}
Ooi, B.~C. \BBOP1987\BBCP.
\newblock \BBOQ Spatial kd-tree: {A} data structure for geographic
  database\BBCQ\
\newblock In Schek, H.\BBACOMMA\  \BBA\ Schlageter, G.\BEDS, {\Bem
  Datenbanksysteme in B{\"{u}}ro, Technik und Wissenschaft, GI-Fachtagung,
  Darmstadt, 1.-3. April 1987, Proceedings}, \lowercase{\BVOL}\ 136 of {\Bem
  Informatik-Fachberichte}, \BPGS\ 247--258. Springer.

\bibitem[\protect\BCAY{{OpenStreetMap contributors}}{{OpenStreetMap
  contributors}}{2017}]{OpenStreetMap}
{OpenStreetMap contributors} \BBOP2017\BBCP.
\newblock \BBOQ {Planet dump retrieved from https://planet.osm.org }\BBCQ\
\newblock \url{ https://www.openstreetmap.org }.

\bibitem[\protect\BCAY{Roussopoulos, Kelley,\ \BBA\ Vincent}{Roussopoulos
  et~al.}{1995}]{DBLP:conf/sigmod/RoussopoulosKV95}
Roussopoulos, N., Kelley, S., \BBA\ Vincent, F. \BBOP1995\BBCP.
\newblock \BBOQ Nearest neighbor queries\BBCQ\
\newblock In Carey, M.~J.\BBACOMMA\  \BBA\ Schneider, D.~A.\BEDS, {\Bem
  Proceedings of the 1995 {ACM} {SIGMOD} International Conference on Management
  of Data, San Jose, California, USA, May 22-25, 1995.}, \BPGS\ 71--79. {ACM}
  Press.

\bibitem[\protect\BCAY{Sellis, Roussopoulos,\ \BBA\ Faloutsos}{Sellis
  et~al.}{1987}]{DBLP:conf/vldb/SellisRF87}
Sellis, T.~K., Roussopoulos, N., \BBA\ Faloutsos, C. \BBOP1987\BBCP.
\newblock \BBOQ The r+-tree: {A} dynamic index for multi-dimensional
  objects\BBCQ\
\newblock In Stocker, P.~M., Kent, W., \BBA\ Hammersley, P.\BEDS, {\Bem
  VLDB'87, Proceedings of 13th International Conference on Very Large Data
  Bases, September 1-4, 1987, Brighton, England}, \BPGS\ 507--518. Morgan
  Kaufmann.

\bibitem[\protect\BCAY{Sturtevant}{Sturtevant}{2012}]{DBLP:journals/tciaig/Sturtevant12}
Sturtevant, N.~R. \BBOP2012\BBCP.
\newblock \BBOQ Benchmarks for grid-based pathfinding\BBCQ\
\newblock {\Bem {IEEE} Trans. Comput. Intellig. and {AI} in Games}, {\Bem
  4\/}(2), 144--148.

\bibitem[\protect\BCAY{Tung, Hou,\ \BBA\ Han}{Tung et~al.}{2001}]{thh-scipo-01}
Tung, A. K.~H., Hou, J., \BBA\ Han, J. \BBOP2001\BBCP.
\newblock \BBOQ {S}patial clustering in the presence of obstacles\BBCQ\
\newblock In {\Bem Proceedings 17th International Conference on Data
  Engineering}, \BPGS\ 359--367.

\bibitem[\protect\BCAY{Xia, Hsu,\ \BBA\ Tung}{Xia
  et~al.}{2004}]{DBLP:conf/bncod/XiaHT04}
Xia, C., Hsu, D., \BBA\ Tung, A. K.~H. \BBOP2004\BBCP.
\newblock \BBOQ A fast filter for obstructed nearest neighbor queries\BBCQ\
\newblock In Williams, M.~H.\BBACOMMA\  \BBA\ MacKinnon, L.~M.\BEDS, {\Bem Key
  Technologies for Data Management, 21st British National Conference on
  Databases, {BNCOD} 21, Edinburgh, UK, July 7-9, 2004, Proceedings},
  \lowercase{\BVOL}\ 3112 of {\Bem Lecture Notes in Computer Science}, \BPGS\
  203--215. Springer.

\bibitem[\protect\BCAY{Zhang, Papadias, Mouratidis,\ \BBA\ Zhu}{Zhang
  et~al.}{2004}]{DBLP:conf/edbt/ZhangPMZ04}
Zhang, J., Papadias, D., Mouratidis, K., \BBA\ Zhu, M. \BBOP2004\BBCP.
\newblock \BBOQ Spatial queries in the presence of obstacles\BBCQ\
\newblock In Bertino, E., Christodoulakis, S., Plexousakis, D., Christophides,
  V., Koubarakis, M., B{\"{o}}hm, K., \BBA\ Ferrari, E.\BEDS, {\Bem Advances in
  Database Technology - {EDBT} 2004, 9th International Conference on Extending
  Database Technology, Heraklion, Crete, Greece, March 14-18, 2004,
  Proceedings}, \lowercase{\BVOL}\ 2992 of {\Bem Lecture Notes in Computer
  Science}, \BPGS\ 366--384. Springer.

\bibitem[\protect\BCAY{Zhao, Taniar,\ \BBA\ Harabor}{Zhao
  et~al.}{2018}]{DBLP:conf/socs/ZhaoTH18}
Zhao, S., Taniar, D., \BBA\ Harabor, D.~D. \BBOP2018\BBCP.
\newblock \BBOQ Fast k-nearest neighbor on a navigation mesh\BBCQ\
\newblock In Bulitko, V.\BBACOMMA\  \BBA\ Storandt, S.\BEDS, {\Bem Proceedings
  of the Eleventh International Symposium on Combinatorial Search, {SOCS} 2018,
  Stockholm, Sweden - 14-15 July 2018}, \BPGS\ 124--132. {AAAI} Press.

\end{thebibliography}
\end{document}